\documentclass[11pt,a4paper]{article}

\usepackage{amssymb,amsmath,amsthm,amsfonts,latexsym}
\usepackage{bbm,mathrsfs}
\usepackage{cite}
\usepackage{graphicx}
\usepackage{algorithm}
\usepackage{algorithmic}
\usepackage{color}
\usepackage{caption}
\usepackage{subcaption}
\usepackage[margin=1in]{geometry}

\newcommand{\beq}{\begin{equation}}
\newcommand{\eeq}{\end{equation}}

\newcommand{\xtest}{x'}
\newcommand{\Xtest}{X'}
\newcommand{\mybf}{}

\newcommand{\algorithmicinput}{\textbf{Input }}
\newcommand{\INPUT}{\item[\algorithmicinput]}

\newcommand{\algorithmictrain}{\textbf{Training phase:}}
\newcommand{\TRAIN}{\item[\algorithmictrain]}
\newcommand{\algorithmictest}{\textbf{Test phase:}}
\newcommand{\TEST}{\item[\algorithmictest]}

\newcommand{\R}{\mathbb{R}}
\newcommand{\N}{\mathbb{N}}

\newtheorem{theorem}{Theorem}[section]
\newtheorem{corollary}[theorem]{Corollary}

\newtheorem{claim}[theorem]{Claim}

\begin{document}

%\begin{frontmatter}
\title{Diffusion Nets}

\author{Gal~Mishne, Uri~Shaham, Alexander~Cloninger and Israel~Cohen%
\thanks{G.~Mishne and I.~Cohen are with the Department of Electrical
Engineering, Technion - Israel Institute of Technology, Haifa 32000, Israel
(email: galga@tx.technion.ac.il; icohen@ee.technion.ac.il).
U.~Shaham is with Department of Statistics, Yale University, New Haven, CT 06520 (email: uri.shaham@yale.edu).
A.~Cloninger is with Department of Mathematics, Yale University, New Haven, CT 06520 (email: alexander.cloninger@yale.edu).
}
}

\maketitle

\begin{abstract}
Non-linear manifold learning enables high-dimensional data analysis, but requires out-of-sample-extension methods to process new data points. 
In this paper, we propose a manifold learning algorithm based on deep learning to create an encoder, which maps a high-dimensional dataset and its low-dimensional embedding, and a decoder, which takes the embedded data back to the high-dimensional space. 
Stacking the encoder and decoder together constructs an autoencoder, which we term a diffusion net, that performs out-of-sample-extension as well as outlier detection. 
We introduce new neural net constraints for the encoder, which preserves the local geometry of the points, and we prove rates of convergence for the encoder.  
Also, our approach is efficient in both computational complexity and memory requirements, as opposed to previous methods that require storage of all training points in both the high-dimensional and the low-dimensional spaces to calculate the out-of-sample-extension and the pre-image.
\end{abstract}
%
%\begin{keyword}
%manifold learning; diffusion maps; deep learning; autoencoder; out-of-sample extension
%\end{keyword}

%\end{frontmatter}

\section{Introduction}
\label{sec:intro}
Real world data is often high dimensional, yet concentrates near a lower dimensional manifold, embedded in the high dimensional ambient space.
In many applications, finding a low-dimensional representation of the data is necessary to efficiently handle it and the representation usually reveals meaningful structures within the data.
In recent years, different manifold learning methods have been developed for high dimensional data analysis, which are based on the geometry of the data, i.e. preserving distances within local neighborhoods of the data. 
These include kernel PCA~\cite{scholkopf1998nonlinear}, ISOMAP~\cite{Tenenbaum:2000}, locally linear embedding (LLE)~\cite{Roweis:2000}, Laplacian eigenmaps~\cite{Belkin2003}, Hessian eigenmaps~\cite{Donoho2003} and Diffusion maps~\cite{Coifman2006}.
Some of these methods are spectral methods, based on the eigenvectors of adjacency or affinity matrices of graphs on the data. 
These methods are capable of capturing a smooth representation of the data and have been shown to be robust to noise and outliers.
The ability to capture the underlying low-dimensional structure of data has made these methods appropriate for dimensionality reduction.
Unlike classical dimensionality reduction methods, such as principal component analysis (PCA), these methods are nonlinear, which is essential as real-world data typically does not lie on a hyperplane. 
In addition, they preserve local structures in the data, while disregarding distances between points that are far apart, which are typically meaningless in high-dimensional data. 
These approaches are very popular in machine learning, signal processing and data mining applications.

When the data set is very large, or when processing online sequential data, it is impractical to directly compute an embedding for the entire dataset.
The computational complexity of calculating the affinity matrix and the eigen-decomposition of the matrix become infeasible in terms of memory and running-time.
Since these non-linear techniques do not provide an explicit mapping from the data to the embedding, out-of-sample extension (OOSE) methods are used to extend the embedding to new data points~\cite{Bengio2003,Fowlkes:2004,Coifman2006a,Lafon2006,laplacianpyramid,PascualRD13Arxiv,Mousazadeh2015,Aizenbud2013}.
In such cases, the low-dimensional embedding is constructed for a representative sample of the data and is then extended to all remaining, or new points.
This is a common approach in image processing applications, for example, especially for high-resolution images~\cite{He2009,Farbman:2010}.

Deep learning has gained popularity in the past years, achieving state-of-the-art results in machine learning, computer vision and speech processing applications, handling increasingly large datasets~\cite{Hinton2006,Hinton2006b,Bengio:2009}. 
Deep neural nets are capable of learning increasingly abstract representations for the data~\cite{Bengio2013}, some of which are robust to small perturbations around training points~\cite{Vincent2008,Rifai2011}.
However, these representations are built globally, without incorporating local geometry or density of the data.  Jia et al.~\cite{Jia2015} recently introduced Laplacian Autoencoders, which impose locality preserving constraints via the weighted affinity matrix. 
Their goal is to improve pre-training of autoencoders in constructing neural networks. 
However, their formulation lacks both a parameter to balance between the reconstruction loss and the affinity regularizer, and an embedding to compare to, so there is no way to ensure that the final representation is anywhere near the true manifold eigenfunctions. 
Because of this, there are no theoretical guarantees on the convergence of the algorithm or on the usefulness of the representations.

In this paper, we propose a new approach to out-of-sample extension, applying a deep neural network to learn the mapping between the data and the embedding. 
We employ deep learning from a manifold learning perspective, by explicitly incorporating a manifold embedding of the data in the deep learning framework.
We address three closely connected problems: OOSE of the embedding to new points, a pre-image solution~\cite{Kwok2004,Arias2007} that can include regularization, and outlier detection on test data.
This third goal is important and often neglected in OOSE methods. 
If new data is an outlier and does not fit the model of the training data, the extension is ill-defined and its new representation will be insufficient.

To accomplish these three goals, we propose to train a neural network-based encoder and decoder, and combine them in an autoencoder. 
First, we train a multi-layer encoder to approximate the low-dimensional embedding on the data; specifically we use diffusion maps~\cite{Coifman2006} for embedding the data.
The encoder performs OOSE of the embedding.
We then train a multi-layer decoder whose input is the diffusion map, to learn the inverse mapping between the embedding and the data.
The decoder enables to recover the pre-image of the embedding, mapping new points in the diffusion space to the high-dimensional data space.
Finally, by stacking the two networks, we obtain an autoencoder, termed the diffusion net.
The diffusion net enables to perform outlier detection, indicating when the extension of a given point via the encoder is faulty due to its being an outlier that does not follow the model of the data.
The diffusion net also performs denoising of the data, reconstructing a clean version of noisy data.
Our approach is both computationally efficient and has low memory costs.
Once the diffusion net has been trained, it is unnecessary to retain the training data and embeddings as required in other methods.
Thus, harnessing the efficiency of deep learning networks enables to efficiently process large quantities of data. 

This paper is organized as follows. Section~\ref{sec:background} provides background on manifold learning and deep neural networks.
In Sections~\ref{sec:diffnet} and ~\ref{sec:implement} we propose a deep learning approach for manifold learning, enabling out-of-sample extension, pre-image computation and outlier detection.
A proof on bounding the convergence of a single layer encoder is presented in Section~\ref{sec:lemma}.
We present experimental results in Section~\ref{sec:results}.
Future research directions are discussed in Section~\ref{sec:future}.

\section{Background and Related Work}
\label{sec:background}
\subsection{Diffusion Maps}
\label{subsec:diffusion}
Diffusion maps is a popular manifold learning technique, based on the construction of the graph Laplacian of the data set~\cite{Coifman2006a}. 
It has been used successfully in various signal processing, image processing and machine learning applications~\cite{Singer:2009,Talmon2012,Lafon2006,Farbman:2010,David2012,Gepshtein:2013,Mishne2013,Haddad2014,Coifman:2014}.
In this section, we briefly review the diffusion maps construction. 
For a detailed discussion on diffusion maps, see~\cite{Coifman2006}.

Given a high dimensional set $X=\{x_i\}_i$, a weighted graph is constructed with the data points as nodes and the weight of the edge connecting two nodes, $k(x_i,x_j)$, $x_i, x_j \in X$, as a measure of the similarity between the two data points. 
The affinity matrix $\mybf{K}[i,j]=k(x_i,x_j)$ is required to be symmetric and non-negative, where a common choice is a radial basis function (RBF) kernel 
\beq
\label{eq:kernel}
k(x_i,x_j)=\exp\left\{-\Vert x_i-x_j\Vert^2/\sigma^2\right\},
\eeq
where $\sigma>0$ is a global scale parameter.
A local scale can be set for each point as in~\cite{selftune}.
In practice, $\mybf{K}$ can be computed using only the nearest neighbors of every point and $\mybf{K}[i,j]$ is set to zero for $x_j$ that are not among the nearest neighbors of $x_i$.

We apply a normalization of the data to obtain an approximation of the Laplace-Beltrami operator on the data, so the embedding will not rely on the distribution of the points~\cite{Coifman2006,Lafon2006}.
The kernel is normalized by the degree of each point $\mybf{D}[i,i]=\sum_{j\in X}k(x_i,x_j)$, by
\begin{equation*}
\label{eq:ln_norm}
\mybf{\widetilde{K}} = \mybf{D}^{-1}\mybf{K}\mybf{D}^{-1}=\sum_{j}\mybf{\widetilde{K}}[i,j].
\end{equation*}
A random walk is then created on the normalized data set by:
\begin{equation}
\label{eq:random_walk}
\mybf{P} = \mybf{\widetilde{D}}^{-1}\mybf{\widetilde{K}}, \;\;\; \mybf{\widetilde{D}}(i,i)=\sum_{j}\mybf{\widetilde{K}}[i,j].
\end{equation}
The row-stochastic matrix $\mybf{P}$ satisfies $\mybf{P}[i,j]\geq0$ and $\sum_{j\in X}\mybf{P}[i,j]=1$ and therefore can be viewed as the transition matrix of a Markov chain on the data set $X$. 
The eigen-decomposition of $\mybf{P}$ yields a sequence of biorthogonal left and right eigenvectors, $\phi_\ell$ and $\psi_\ell$ respectively, and a sequence of positive eigenvalues: $1 = |\lambda_0|\geq|\lambda_1|\geq ...$.
Then, $t$ steps of the Markov chain can be calculated as 
\begin{equation*}
\label{eq:eigen_decompose}
p_t(x_i,x_j)=\sum_{l\geq0} \lambda^t_\ell\psi_\ell(x_i)\phi_\ell(x_j).
\end{equation*}

A diffusion distance $d_{\textrm{DM}}(x_i,x_j;t)$ between two points $x_i,x_j \in X$ is defined by
\begin{equation}
\label{eq:diffusion_distance1}
d_{\textrm{DM}}(x_i,x_j;t ) = \sum_{x_k\in X}\frac{\big(p_t(x_i,x_k)-p_t(x_j,x_k)\big)^2}{\phi_0(x_k)} = \sum_{l\geq1} \lambda^{2t}_\ell(\psi_\ell(x_i)-\psi_\ell(x_j))^2.
\end{equation}
where $\phi_0$ is the stationary probability distribution on the graph.
This metric is robust to noise, since the distance between two points depends on all possible paths of length $t$ between the points.
Due to the spectrum decay of $\mybf{P}$, the diffusion distance can be approximated using only the first $d$ eigenvectors.
Equation~(\ref{eq:diffusion_distance1}) implies that a mapping can be defined between the original space and the eigenvectors $\psi_\ell$.
Retaining only the first $d$ eigenvectors, the mapping $\Psi_t$ embeds the data set $X$ into the Euclidean space $\mathbb{R}^{d}$, where the diffusion distance is equal to the Euclidean distance in this new embedding:
\begin{equation}
\label{eq:diffusion_map}
\Psi_t:x_i\rightarrow \big( \lambda_1^t\psi_1(x_i), \lambda_2^t\psi_2(x_i),..., \lambda_{d}^t\psi_{d}(x_i)\big)^T .
\end{equation}
Note that $\psi_0$ is not used in the embedding because it is a constant vector.
In this paper we set $t=1$, but our approach can be used for estimating the embedding for general $t$.

\subsection{Out-of-sample Function Extension}
\label{subsec:extend}
Having calculated a diffusion map $\Psi$ on the data $X$, various methods have been proposed for extending $\Psi$ to new points.
In simple examples, there are analytic mechanisms for creating harmonic extensions when the eigenfunctions of the Laplacian can be derived analytically.
However, this is not applicable in the general case.
The Nystr\"{o}m extension method is a popular method for general OOSE. 
Given a new point $\xtest \in \mathcal{M} \setminus X$, the eigenvector $\psi_\ell$ is extended to this point as:
\beq
\label{eq:nystrom}
 \widehat{\psi}_\ell(\xtest) = \frac{1}{\sqrt{\lambda_\ell}} \sum_{j=1}^m p(\xtest,x_j){\psi}_\ell(x_j), \;\;\; \ell = 1, \cdots, d.
\eeq
Geometric Harmonics~\cite{Coifman2006a,Lafon2006} is an OOSE method which improves upon the Nystr\"{o}m extension method.
It treats both the numerical instability due to $\lambda$ by extending only the eigenvectors with significant eigenvalues.
In addition, it finds an appropriate scale for the kernel in the extension, dependent on the function that is being extended.
Rabin and Coifman proposed a Laplacian pyramids-based OOSE method in~\cite{laplacianpyramid}. 
The eigenvectors are extended in an iterative multi-scale scheme, where the number of scales is adapted to the complexity of each eigenvector separately. 
This approach was recently extended in~\cite{PascualRD13Arxiv} to implicitly incorporate cross validation in the training procedure and avoid over-fitting in the training.
Aizenbud, Bermanis and Averbuch~\cite{Aizenbud2013} introduced an extension method based on a generalized least squares solution for each new test point within its local neighborhood in the training set. 
This solution is shown to minimize the Mahalanobis distance between the embedding of the training points and the estimated embedding for the new point, in respect to a covariance matrix that incorporates geometric properties of the data and embedding.

The computational complexity of these methods typically depends on the number of training points, since these methods rely on the distances between a new test point and all training points, or its nearest neighbors in the training set (determined by a nearest neighbor search algorithm).
Therefore, to perform OOSE, it is necessary to keep all the training points and their corresponding embeddings in memory.
If the affinity matrix $\mybf{K}$ in~(\ref{eq:random_walk}) is based on non-Euclidean local metrics, such as~\cite{Singer2008, Talmon2012b, Talmon2013, Haddad2014,Mishne2014b}, it is not possible to use nearest neighbors search, since each training point is associated with its own local metric.
This increases the complexity of the distance calculation and adds to the memory requirements of the OOSE.
The method we propose has no such requirement.
After training, our approach enables OOSE without retaining any of the training data or embeddings.

\subsection{Deep Neural Networks}
\label{subsec:deeplearn}
Artificial neural networks (ANNs) are networks composed of interconnected computational units termed \emph{neurons}, which are typically organized in layers.
A deep neural network is composed of multiple hidden layers, and is typically designed as a feed-forward  network, in which there are no cycles or loops.
In our framework, we use multi-layer perceptrons (MLP), which are a popular and important class of neural nets, in which the layers are densely connected.
The output of each layer is computed as an affine mapping of the previous layer followed by a non-linear function:
\begin{equation*}
a^{(l+1)} = \sigma(\mybf{W}^{(l)}a^{(l)}+b^{(l)}),
\end{equation*}
where $a$ is termed the activation, $\mybf{W}^{(l)}[i,j]$ denotes the weight associated with the connection between unit $j$ in layer $l$ and unit $i$ in layer $l+1$, $b^{(l)}$ is a bias vector, $\sigma(\cdot)$ is a non-linear function applied element-wise, and $a^{(1)} = X$. 
We denote the number of hidden units in layer $l$ by $s_l$ and the overall number of layers in a network, including the input and output, by $L$.
In our experiments, we used a sigmoid non-linearity in the activation: $\sigma(z) = \frac{1}{1+e^{-z}}$.
Other choices include $\sigma(z)=\tanh(z)$ and rectified linear units: $\sigma(z)=\textrm{ReLU}(z)=\max\{0,z\}$.

Deep nets have been successfully applied to various tasks such as regression for learning a function over a given dataset, classification, feature learning, etc., achieving state-of-the-art results.
The task the network performs is determined by the output layer and the cost function minimized over the network.
In supervised learning, the goal is to predict a function or labels on the input data.
The cost function of a network consists of a loss function, and a weight regularization term is usually added in order to prevent over-fitting.
Given a multi-layer network, we denote the weights and biases of all the layers by $\theta = \{\mybf{W}^{(l)}, b^{(l)} \}_l$.
For regression of a multi-dimensional function, $y \in \mathbb{R}^d$, as in our application, the squared error can be used for the loss term:
\begin{equation}
\label{eq:regcost}
J(\theta) = \frac{1}{2m}\sum_{i=1}^m \Vert o(x_i;\theta) - y_i \Vert^2  + \frac{\mu}{2} \sum_{l=1}^{L-1}  \Vert \mybf{W}^{(l)} \Vert_F^2, 
\end{equation}
where $o(x_i)\in \mathbb{R}^d$ is the output of the net for input $x_i$, $\Vert \cdot \Vert_F^2$ is the Frobenius norm, $\mu$ is a cost parameter.
We use an $l_2$ penalty on the weights, which is a very popular choice for regularization.
Neural nets are typically trained using variants of stochastic gradient descent (SGD) for calculating the weight and biases in the network that minimize the cost function. 
The gradient of the loss function relative to the weights is computed efficiently using backpropagation~\cite{Rojas1996}, starting from the output layer backwards to the input.

\subsubsection{Autoencoders}
Deep learning can also be used in an unsupervised manner, such as in training autoencoders.
An autoencoder is composed of an encoder function $f(\cdot)$ and a decoder function $g(\cdot)$, where the dimension of $f$ is typically smaller that the dimension of the input data. 
The reconstruction of an input $x_i$ is given by stacking the decoder on the encoder: $r(x_i) = g(f(x_i))$, and the autoencoder is trained to minimize a reconstruction error loss $L(x, r(x))$ over the training points, i.e. trying to approximate the identity function.
This is performed by setting the number of output units of the decoder to be the dimension of the input data, and fine-tuning a loss function between the output and the input.
The output of the encoder, i.e. the middle layer of the autoencoder, can be seen as a low-dimensional representation of the data.
Autoencoders were popularized by Hinton~\cite{Hinton2006b}.
Basic autoencoders consist of a single hidden layer~\cite{Vincent2008,Rifai2011}, but deep autoencoders have also been proposed for classification, denoising, image retrieval, speech enhancement and more~\cite{Hinton2006b, Vincent2008}.

\section{Diffusion Net}
\label{sec:diffnet}
\subsection{Problem Setup}
We assume the data lies on a smooth, compact, $d$-dimensional Riemannian manifold $\mathcal{M}$ embedded in a high-dimensional space $\mathbb{R}^n$, ($d<n$).
Calculating the diffusion map $\Psi : X \rightarrow \mathbb{R}^d$~(\ref{eq:diffusion_map}) for a given training set $X=\{x_i\}_{i=1}^m$ maps the high-dimensional space to Euclidean space $\mathbb{R}^d$, revealing the low-dimensional structure of the data.
We address three related problems in this setting: (a) Out-of-sample extension, (b) pre-image solution and (c) outlier detection.

Given new test points $\Xtest=\{\xtest\} \subseteq \mathcal{M} \setminus X$, the purpose of an out-of-sample extension method is to extend $\Psi$ to the new points.
For the given training points, the extension should be as close as possible to the true embedding: $\Vert \widehat{\Psi}(x_i)- \Psi(x_i)\Vert_2 < \epsilon$, for small $\epsilon$.
In addition, the extension $\widehat{\Psi}(\xtest)$ should preserve the properties of the original embedding, such as preserving local structures in the data.

The second problem we address is calculating the pre-image~\cite{Kwok2004,Arias2007}. 
Given a point $\varphi$ in the diffusion space, the pre-image of $\varphi$ is a data-point $x$ for which $\Psi(x) = \varphi$.
The pre-image has been shown to be closely related to the OOSE problem~\cite{Arias2007}.
Calculating the pre-image is essential to being able to pull back calculations made in the embedding space to the data space.

Finally, we aim to provide a new outlier detection measure, to detects outlier in newly arrived data.
In OOSE algorithms, the ability to provide a good embedding for new points depends on how well they follow the model of the training data, i.e. their distance from the training data.
For example in kernel-based extension methods such as Nystr\"{o}m, the scale of the Gaussian kernels limits how close points need to be to the training data. 
Denote by $x^* = \underset{x_i \in X}{\arg \min} \Vert \xtest - x_i \Vert$ the nearest-neighbor of a test point $\xtest$ within the training set $X$.
If the distance between the two is very large, $\Vert \xtest - x^* \Vert \gg \sigma$, then $\widehat{\Psi}(\xtest) = \Psi(x^*)$, i.e. the OOSE reverts to 1-nearest neighbor prediction.
Numerically, if $\Vert \xtest - x^* \Vert / \sigma \rightarrow 0$, then the affinity to the training set evaluates to zero, and the OOSE is the origin.
In related works (e.g.~\cite{Fowlkes:2004,Coifman2006a,Lafon2006,laplacianpyramid}), there is typically an implicit assumption that the test data follows the distribution of the training data.
However, the embedding obtained for outliers is uninformative and can lead to mistakes in the task the embedding is used for, i.e. classification, signal processing, etc.
Therefore, indicating if a point is an outlier within an OOSE framework is important to properly processing the data.

\subsection{Our Approach}
We propose a solution to these problems based on deep neural nets, training a neural network-based encoder and decoder, and combining them in a deep multi-layer autoencoder. 
Thus, instead of training a general autoencoder in which the middle layer can be seen as providing a low-dimensional representation of the data, we incorporate the given embedding on the data in the training procedure.
The encoder and the decoder are both MLPs composed of $L$ layers, whose output layer is a regression layer.
In our approach, the encoder learns a mapping from the manifold to the diffusion space, by minimizing the squared loss between the output of the encoder and $\Psi(x)$. 
In addition, we impose a new constraint to preserve properties of the embedding.
The decoder learns the inverse mapping from the low-dimensional diffusion space back to the high-dimensional space of the data, solving the pre-image problem.
By stacking the encoder and decoder we obtain an autoencoder, whose inner-most layer computes the diffusion embedding of the data. 
The autoencoder can be used for both outlier detection and denoising.
Our framework is presented in Algorithms~\ref{alg:encoder},~\ref{alg:decoder} and~\ref{alg:ae}. 
\begin{algorithm}[ht]
\caption{Out-of-sample Extension}
\label{alg:encoder}
\algsetup{indent=2em}
\begin{algorithmic}[1]
\item[\textbf{Encoder} (Section~\ref{subsec:encoder})]
\TRAIN
\INPUT Training set $X$, diffusion embedding $\Psi$, number of layers for encoder $L$.
\STATE Initialize the weights $\theta^e = \{\mybf{W}^{(l)},b^{(l)}\}_l$ of the encoder and set $a^{(1)} = X$, $y = \Psi$.
  \STATE Optional: pre-train every hidden layer in every stack as an autoencoder. (Section~\ref{subsec:pretrain})
 \STATE Optimize the parameters of the network by fine-tuning cost (\ref{eq:encoderEV}) with back-propagation. (Section~\ref{subsec:opt})
\TEST
\INPUT test data $\Xtest$	
\STATE Calculate the out-of-sample-extension to new points as the output of the encoder $\widehat{\Psi}(\xtest)= o^e(\xtest)$.
\end{algorithmic}
\end{algorithm}

\begin{algorithm}[ht]
\caption{Pre-Image}
\label{alg:decoder}
\algsetup{indent=2em}
\begin{algorithmic}[1]
\item[\textbf{Decoder} (Section~\ref{subsec:decoder})]
\TRAIN
\INPUT Training set $X$, diffusion embedding $\Psi$, number of layers for decoder $L$.
\STATE Initialize the weights $\theta^d=\{\mybf{W}'^{(l)},b'^{(l)}\}_l$ of the decoder and set $a^{(1)} = \Psi$, $y = X$.
  \STATE Optional: pre-train every hidden layer in every stack as an autoencoder.
 \STATE Optimize the parameters of the network by fine-tuning cost (\ref{eq:decoder}) with back-propagation.
 \TEST
 \INPUT test points in the embedding space $\{\varphi\}$	
\STATE Calculate the pre-image of the test points as the output of the decoder $\widehat{x}= o^d(\varphi)$.
 \end{algorithmic}
\end{algorithm}

 \begin{algorithm}[ht]
\caption{Autoencoder}
\label{alg:ae}
\algsetup{indent=2em}
\begin{algorithmic}[1]
 \item[\textbf{Autoencoder} (Section~\ref{subsec:autoencoder})]
\TRAIN
 \STATE Stack decoder trained as in Alg.~\ref{alg:decoder} on top of encoder trained as in Alg.~\ref{alg:encoder} and obtain autoencoder with $2L-1$ layers.
  \STATE Calculate average reconstruction error $\epsilon$.
 \TEST
 \INPUT test data $\Xtest$
\STATE Calculate reconstruction of the point as the output of the autoencoder $\widehat{\xtest}=r(\xtest)$.
\STATE Calculate an outlier detection score as threshold on the reconstruction cost $\Vert r(\xtest) -\xtest \Vert > C\epsilon$.
\end{algorithmic}
\end{algorithm}

\subsection{Encoder}
\label{subsec:encoder}
The encoder learns the mapping between the high-dimensional data space and the embedding space.
The encoder is designed as an MLP, minimizing the $L_2$ loss between the diffusion map $\Psi$ and the output of the net, i.e. $o^e(x_i)=\Psi(i)$ in~(\ref{eq:regcost}), where $o^e$ denotes the output matrix of the encoder for all training points, $o^e\in\mathbb{R}^{d\times m}$.
To facilitate the output of the encoder approximating the diffusion map, we add a new constraint to the training objective of the encoder.
Since the coordinates of the diffusion map are eigenvectors of the random walk matrix on the data,$\mybf{P}$ in~(\ref{eq:random_walk}), we add a new cost term on the output of the encoder being an eigenvector of this matrix.
The additional term, termed the eigenvector (EV) constraint, is
\begin{equation}
\label{eq:evconstraint}
J^{\textrm{EV}} = \frac{\eta}{2m}\sum_{j=1}^{d} \Vert (\mybf{P} - \lambda_j \mybf{I}_{m\times m}) (o^e_j)^T \Vert^2,  
\end{equation}
where $o^e_j$ is the $j$-th row of the output matrix, $\lambda_j$ is the $j$-th eigenvalue of the affinity matrix, and $\eta$ is an optimization cost parameter.

This new term imposes smooth locality on the training points, and maintains the local geometry of the points in the encoder function. 
In contrast to general constraints in deep learning, this constraint ``mixes'' output values of different data points together, via the random-walk matrix.
The EV constraint enforces that an output value for a given data point can be reconstructed as a weighted average of the outputs of the local neighborhood of the data point.
The locality of the neighborhood depends on the locality of the kernel used to define the affinity matrix~(\ref{eq:kernel}). 
This constraint serves to minimize the loss error on the output, such that the output is not a general regression of the desired function, but an eigenvector of the matrix, thus maintaining the geometric properties of the embedding.
Note that for a high value of $\eta$, this constraint forces the output to the origin, which is a trivial solution to minimizing this constraint.

Although we use diffusion maps for manifold learning, other kernel methods may be used and imposed with the EV constraint.
This only requires replacing the matrix $\mybf{P}$ with the matrix that is decomposed in other approaches, for example the normalized affinity matrix $\mybf{D}^{-1/2}\mybf{K}\mybf{D}^{-1/2}$ as in spectral clustering \cite{Weiss1999} or the unnormalized graph Laplacian matrix $\mybf{L}=\mybf{D}-\mybf{K}$ as in Laplacian Eigenmaps~\cite{Belkin2003}.

Adding the EV constraint to the cost function of the encoder, the total cost is:
\begin{equation}
\label{eq:encoderEV}
J^e(\theta^e) = \frac{1}{2m}\sum_{i=1}^m \Vert o^e(x_i) - \Psi(x_i) \Vert^2  + \frac{\mu}{2} \sum_{l=1}^{L-1}  \Vert \mybf{W}^{(l)} \Vert_F^2 
 + \frac{\eta}{2m}\sum_{j=1}^{d} \Vert (\mybf{P} - \lambda_j \mybf{I}) (o^e_j)^T \Vert^2,
\end{equation}
where $\theta^e = \{\mybf{W}^{(l)},b^{(l)}\}_l$ denotes the set of the weights and biases of all the layers of the encoder.
To incorporate this constraint in the gradient calculation in back-propagation, the gradient of this constraint in regards to the output layer is:
\begin{equation*}
\label{eq:ev_backprop}
\nabla J^{\textrm{EV}} = \frac{\eta}{m} \sum_{j=1}^d o^e_j \left( \mybf{P}^T - \lambda_j \mybf{I} \right) \left( \mybf{P} - \lambda_j \mybf{I} \right). 
\end{equation*}

For new data points, the encoder performs out-of-sample extension, based on the mapping learned from the training points. 
This extension is computationally efficient as it relies on a few affine transforms and non-linear element-wise functions.
In addition, once trained, it does not depend at all on the training data, making it efficient also in terms of memory.

\subsection{Decoder}
\label{subsec:decoder}
The decoder has the same architecture as the encoder, only in reverse.
The input layer is $a^{(1)} = \Psi$, with size $s_1=d$, and $y=X$ with size $s_{L}=n$.
There are no new constraints in the decoder cost, and it consists of a loss term and weight regularization term as in~(\ref{eq:regcost}):
\begin{equation}
\label{eq:decoder}
J^d(\theta^d)  = \frac{1}{2m}\sum_{i=1}^m  \Vert 
o^d(i) - x_i \Vert^2
 + \frac{\mu}{2} \sum_{l=1}^{L-1} \Vert \mybf{W}'^{(l)} \Vert_F^2 , 
\end{equation}
where where $\theta^d = \{\mybf{W}'^{(l)},b'^{(l)}\}_l$ denotes the set of the weights and biases of all the layers of the decoder, and $o^d(i)\in \mathbb{R}^n$ is the output for input  $\Psi(x_i)$.
Note that the decoder weights $\{\mybf{W}'^{(l)}\}_l$ are not tied to those of the encoder $\{\mybf{W}^{(l)}\}_l$, therefore the number of units in the hidden layers can be different than for the encoder.
Enforcing tied weights in the autoencoder will be explored in future work.

The decoder solves the pre-image problem and has several interesting applications.
As the decoder learns the non-linear inverse mapping between the data and the diffusion space, it enables to output points in the data space given new points in the diffusion space.
This enables data visualization, performed by covering the diffusion space with new points, and inputting these points to the decoder.
This also enables to perform data augmentation for increasing the training set in machine learning applications.
Another benefit of the decoder is the ability to perform calculations in the embedding space and then pull back to the data-space, for example, calculating centroids in a clustering application, or interpolation of points in the embedding space.
In our current formulation we do not impose constraints on the decoder, however, in future work we will explore including a PDE regularizer in the cost function, such as a harmonic constraint.
This will enable recovering interesting surfaces, such as ``minimal surfaces'' from the embedding.

\subsection{Autoencoder}
\label{subsec:autoencoder}
Having trained the encoder and the decoder, the two networks are stacked together to obtain an autoencoder.
This architecture is displayed in Figure~{\ref{fig:network}}. 
\begin{figure*}[t]
\centering
	\includegraphics[width=0.8\linewidth]{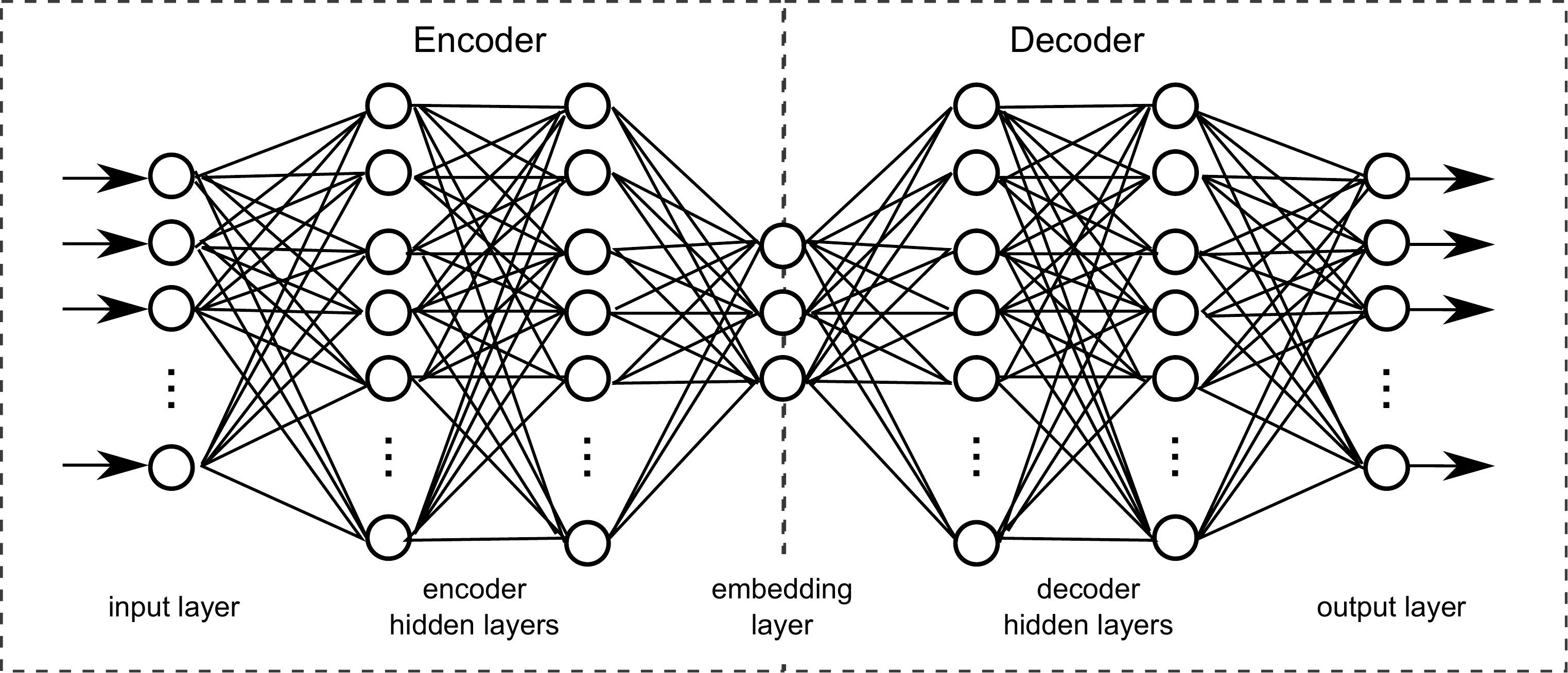}
	\caption{Autoencoder. Left: encoder with two hidden layers, Right: decoder with two hidden layers. The output of the encoder is a prediction of the diffusion map. The output of the decoder is a reconstruction of the data.}
		\label{fig:network}
\end{figure*}
The network is composed from two stacks, one for the encoder and one for the decoder. 
In most of our experiments each stack has 2 hidden layers, and in some we use simpler stacks of a single hidden layer. 

One application of the autoencoder is to use the autoencoder reconstruction error to detect outliers in the data.
Denoting the output of the autoencoder by $r$, the training data provides an average reconstruction error:
\begin{equation}
\label{eq:ae_MSE}
\epsilon = \frac{1}{m} \sum_{i=1}^m \Vert x_i - r(x_i)\Vert^2.
\end{equation}
This enables setting a threshold on new test points to determine whether they fit the model of the data learned by the autoencoder.
If the reconstruction error for a new point $\xtest$ is greater than this threshold, the point is considered an outlier and its OOSE can be disregarded.
This condition is given by
\begin{equation*}
\label{eq:outlier}
\Vert r(\xtest) - \xtest \Vert > C\epsilon,
\end{equation*}
where $C$ is a constant determined by the user and the training data.
Performing outlier detection within the framework of OOSE was previously addressed in~\cite{Aizenbud2013}, where a Mahalanobis distance in the embedding space was used for anomaly detection.
Our approach differs in that our measure depends on the reconstruction of the point in the data space and not in the embedding space.

A second application of the autoencoder is denoising.
When applied to noisy data, the diffusion map recovers a smooth manifold.
Noise in the data, which relates to the relaxation time in the diffusion process, is typically manifested in eigenvectors relating to small eigenvalues~\cite{Singer:2009}.
These are are usually disregarded when choosing the number of dimensions to retain in the embedding.
Thus, noisy points which relate to the same original clean point are embedded at identical coordinates, and this mapping is learned by the encoder.
The decoder, which learns the inverse mapping from the embedding to the data, can only recover a smooth version of the data.
Variations in the recovered data are such that they are along the low-dimensional manifold, whereas variations due to noise, that are not "seen" by the manifold, are suppressed.
Thus, inputting new noisy points into the autoencoder, outputs a clean version of the points.
Vincent et al.~\cite{Vincent2008} provide an interpretation of their denoising autoencoder based on the manifold assumption, however they do not explicitly incorporate the manifold in their training procedure.
In addition, they are intentionally corrupting the data by adding noise to the input and then training the network to recover the clean data.
We do not assume the data is clean, but that the embedding provides a smooth representation of the data.

\section{Practical Considerations}
\label{sec:implement}
This section provides implementation details regarding the optimization of the diffusion net and the computational complexity of our approach.
\subsection{Pre-training}
\label{subsec:pretrain}
In general neural nets, a greedy layer-wise pre-training procedure was first developed by Hinton, Osindero and Teh~\cite{Hinton2006} for Restricted Boltzman Machines, and later extended to autoencoders~\cite{Bengio:2009}.
This procedure was shown to improve the optimization of the neural network.
In this approach, each layer is first trained in an unsupervised manner as an autoencoder whose input is the output of the previous layer. 
This enables to ``initialize'' the parameters $\theta = \{ \mybf{W}^{(l)},b^{(l)}\}_l$ of the network before performing backpropagation, termed fine-tuning, over the complete network for the supervised learning task.
It was shown empirically that this procedure yields improved results compared to initializing with random weights which can get stuck in poor local minimum solutions.

Recently, deep networks trained using very large quantities of labeled data have achieved successful results, with pre-training having very little impact, if at all. 
It appears that employing large amounts of data in the training cancels the advantage gained by pre-training~\cite{Bengio2013}.
In our case, the amount of the data used to train the network is not large, therefore pre-training is beneficial.

We pre-train every hidden layer as a denoising autoencoder~\cite{Vincent2008} with a sparsity term.
This term encourages the average activation of every hidden unit to be small so that the hidden representation of the data is sparse~\cite{Le2011}. 
The loss function is given by
\begin{equation}
\label{eq:sparsity}
J(\theta) = \frac{1}{2m}\sum_{i=1}^m \Vert \hat{\widetilde{x}}_i - x_i \Vert^2  
+ \frac{\mu}{2} \sum_{l=1}^{2} \Vert \mybf{W}^{(l)} \Vert_F^2 
+  \beta \sum_{j=1}^{s_1} \textrm{KL}(\rho \Vert \hat{\rho}_j),
\end{equation}
where $\theta = {\mybf{W}^{(1)},\mybf{W}^{(2)},b^{(1)},b^{(2)}}$.
The input point is $x_i$, $\widetilde{x}$ is $x$ corrupted by setting a random subset of entries to zero and $\hat{\widetilde{x}}_i$ is the reconstruction of $x_i$ by the autoencoder.
The third term is the sparsity loss where where $s_1$ is the number of hidden units, $\rho$ is a sparsity parameter, typically a small value close to zero, $\hat{\rho}_j$  is the average activation of hidden unit $j$ and $\textrm{KL}(\rho \Vert \hat{\rho}_j)$ is the Kullback-Leibler divergence between the probability mass functions of Bernoulli random variables with parameters $\rho$ and $\hat{\rho}_j$.
This constraint imposes that each hidden units ``responds" only to specific input patterns in the data.

\subsection{Optimiztion}
\label{subsec:opt}
The cost function in neural networks is highly non-convex. 
However, convex optimization methods are used, as typically a local minimum yields good results.
We perform training with Limited memory BFGS (L-BFGS)~ \cite{Wright1999,Le2011} with line search procedure, a quasi-Newton method implemented in Matlab's \textit{minFunc} package, and the gradients are computed using standard back-propagation~\cite{Rojas1996}.
L-BFGS is a suitable approach when the dataset is not too large, and even out-performs SGD, which is a more popular approach in deep learning, in certain settings~\cite{Le2011}.

The weights and biases in all layers are initialized with random values from a normal distribution with zero mean and small variance.
The weight regularization parameter $\mu$ was set to be small, i.e. order of $10^{-7}$ or  $10^{-10}$.
In the experimental results in Section~\ref{sec:results} we examine the performance of the new constraint~(\ref{eq:evconstraint}) for different values of $\eta$.

\subsection{Mini-batch training}
Mini-batch training is typically used with large datasets in deep learning and has been shown to improve performance.
In this method, a random subset of samples is chosen for each iteration (or number of iterations) and the gradients are computed and averaged only for this small subset.
Since the constraint we added~(\ref{eq:evconstraint}) entails multiplying the output of the encoder by the normalized random walk matrix, it is essentially ``mixing" between the outputs of different training points.
This is problematic in mini-batch training, since a point's local neighborhood which most influences the eigenvector constraint will not necessarily be included in the mini-batch.
If extracting only a subset of points, a new matrix needs to be calculated for the mini-batch by extracting the relevant rows and columns, and renormalizing the new matrix. 
In this case, this constraint is reduced to the output being only an approximation of the eigenvector of the full random walk matrix.
However, since the size of training data used in manifold learning is usually not large due to computational complexity constraints, it is unnecessary to train using mini-batches.
In this our setting differs from typical large-scale applications in deep learning such as Imagenet, CIFAR-100, etc. 

\subsection{Computational complexity}
Training is an iterative sequential process, performed offline.
Once the encoder has been trained, out-of-sample extension on new data is calculated with complexity $\mathcal{O}(\sum_{l=1}^{L-1}s_l s_{l+1})$ where $s_1 = n$ and $s_{L}=d$. 
In our experiments, we used an encoder with two hidden layers so the complexity was of order $\mathcal{O}(ns_2 + s_2s_3 +s_3d)$.
This is opposed to other OOSE methods in which the complexity depends on $m$, the number of training points, where typically, $m \gg n > d$.
For example, the complexity of OOSE using Nystr\"{o}m is $\mathcal{O}$((d+n)m).

An advantage of our method is that there is no need to retain the training points and embedding, once the network has been trained.
Only the  weight matrices and bias vectors of all the layers of the network are necessary.
Thus, our approach requires memory on the order of $\mathcal{O}(\sum_{l=1}^{L-1} s_l s_{l+1})$.
Other methods, on the other hand, require retaining all training points and embeddings.
For Nystr\"{o}m and Geometric Harmonics~\cite{Coifman2006a} this results in memory on order of $\mathcal{O}((d+n)m)$. 
The memory cost of the PCA-based approach in~\cite{Aizenbud2013} is higher, requiring an additional $\mathcal{O}(md^2)$ to save covariance matrices for the embeddings of the training points.
If using a non-Euclidean metric as in~\cite{Singer2008,Talmon2013}, retaining the covariance matrices of the training points leads to an additional memory cost on the order of $\mathcal{O}(mn^2)$.
Thus, our method has a large advantage in applications and systems in which the memory and run-time are limited.

\section{Bounding the Out-of-sample-extension Error}
\label{sec:lemma}
In this section, we provide a theoretical bound on the error rate for approximating eigenfunctions of the Laplacian using a single layer network of sigmoid hidden units.
The full derivation is given in~\ref{app:proof}.
Suppose $\cal{M}$ is a smooth compact d-dimensional submanifold of $\R^n$.  Assume $\cal{M}$ is equipped with a metric $\rho:\cal{M}\times\cal{M}\rightarrow\R^+$ which is locally bi-Lipschitz with respect to the Euclidean metric, meaning $\exists \delta,\epsilon$ such that 
\begin{eqnarray*}
(1-\epsilon)\|x-y\|_2 \le \rho(x,y) \le (1+\epsilon)\|x-y\|_2, & \textnormal{ for } \|x-y\|<\delta.
\end{eqnarray*}

\begin{theorem}\label{thm:convergence}
Let $\cal{M}$ be a smooth compact submanifold of $\R^n$ equipped with a metric which is locally bi-Lipschitz with respect to the Euclidean metric.  Let $B_r$ be a Euclidean ball of radius $r$ such that $\cal{M}\subset$ $B_r$.  Let $\psi$ be an eigenfunction of the Laplacian of $\cal{M}$ with eigenvalue $\lambda$, and let $f:\R^n\rightarrow \R$ be an extension of the eigenfunction to the ambient dimension via
\begin{eqnarray}\label{eq:extender}
f(x) = exp(-\lambda \|x - P_{\cal{M}} x\|_2^2)\psi(P_{\cal{M}} x), & \textnormal{where } P_{\cal{M}} x = \arg\min_{y\in\cal{M}} \|x-y\|_2.
\end{eqnarray}
Then there is a linear combination $f_K$ of $K$ sigmoidal units such that
\begin{eqnarray}
\left(\int (f(x) - f_K(x))^2 dx \right)^{\frac{1}{2}} \le \frac{C}{\sqrt{K}}.
\end{eqnarray}
\end{theorem}

\begin{corollary}\label{corr:convergence}
Under the same conditions as Theorem \ref{thm:convergence}, let $\psi_1, ..., \psi_d$ be eigenfunctions of the Laplacian and $f_i$ be the extension of $\psi_i$.  Let $f(x) = (f_1(x), ..., f_d(x))$.  Then there exists a single hidden layer network with $Kd$ sigmoidal units and output $o(x)\in\R^d$ such that 
\begin{eqnarray}
\label{eq:extender2}
\left(\int (f(x) - o(x))^2 dx \right)^{\frac{1}{2}} \le \frac{C}{\sqrt{K}}.
\end{eqnarray}
\end{corollary}
Note that Theorem~\ref{thm:convergence} guarantees existence of a solution, yet since the optimization problem is non convex, it does not guarantee convergence.

\section{Experimental Results}
\label{sec:results}
In this section, we present experimental results for several toy problems and real image data.
We demonstrate the performance of the encoder and decoder separately and then joining them in an autoencoder.
We evaluate the effect of the eigenvector constraint and demonstrate how it improves the performance of our system, especially in noisy scenarios.
Finally, the autoencoder is successfully used in outlier detection in images. 
This demonstrates that our solution not only performs OOSE, but also means by which to verify that new data agrees with the model inferred from the training data.
 
\subsection{Encoder}
\label{res:encoder}
Our first example is a closed 3-D curve parametrized by:
\begin{equation}
\label{eq:curve3d}
x_i = \left( \cos(\theta_i), \sin(2\theta_i) , \sin(3\theta_i)\right)^T, \;\;\; \theta_i \in (0,2\pi).    
\end{equation}
We examine the effect of adding the eigenvector constraint~(\ref{eq:evconstraint}) to the encoder cost function~(\ref{eq:encoderEV}) for this toy problem.
We address both the effect of the architecture of the network, i.e. number of layers and units, and the effect of noise in the data.
We add white Gaussian noise $\nu_i$ with standard deviation (std.) $\sigma_\nu=0.05$ to the data:
\begin{equation*}
\label{eq:noisyCurve3d}
\tilde{x}_i = x_i + \nu_i,\;\;\; \nu \sim \mathcal{N}(0,\sigma_\nu^2\mybf{I}_{3 \times 3})    
\end{equation*}
and sample 2000 points from this noisy curve. 
The 2D diffusion map of these points is a smooth circle (Fig.~\ref{fig:noisy_data}). 
\begin{figure}[th]
\centering
	\includegraphics[width=0.5\linewidth]{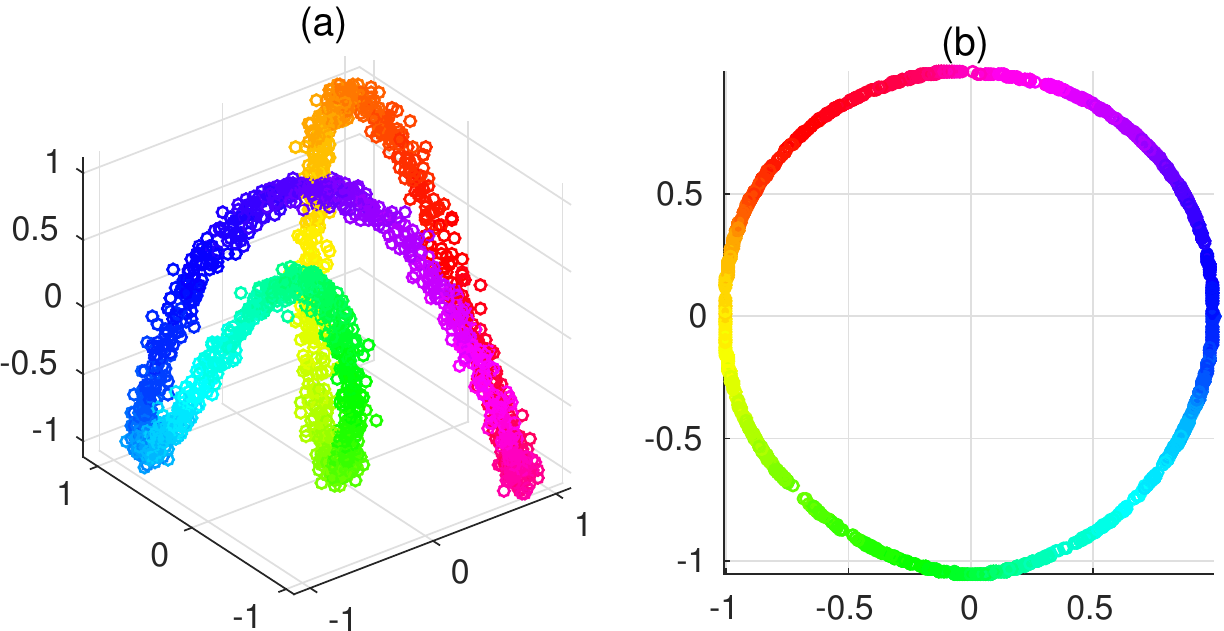}
	\caption{(a) 3D closed curve. (b) First two coordinates of the diffusion map.
	Points colored by $\theta_i$.}
		\label{fig:noisy_data}
\end{figure}

\begin{figure}[th]
\centering
\begin{subfigure}[b]{0.45\textwidth}
	\includegraphics[height=2in]{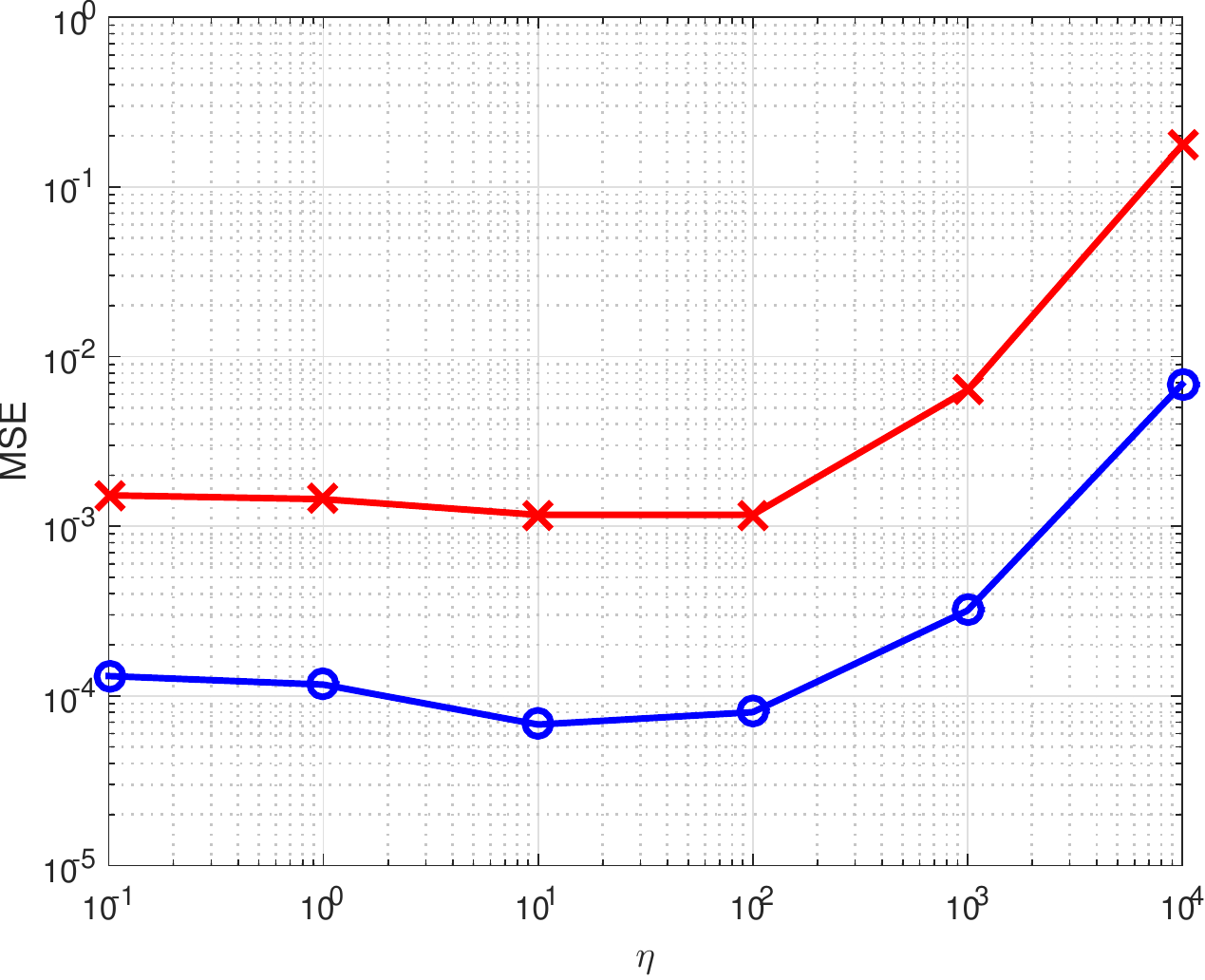}
	\caption{}
		\label{fig:dm_mse_1vs2}
	\end{subfigure}
	\begin{subfigure}[b]{0.45\textwidth}
	\includegraphics[height=2in]{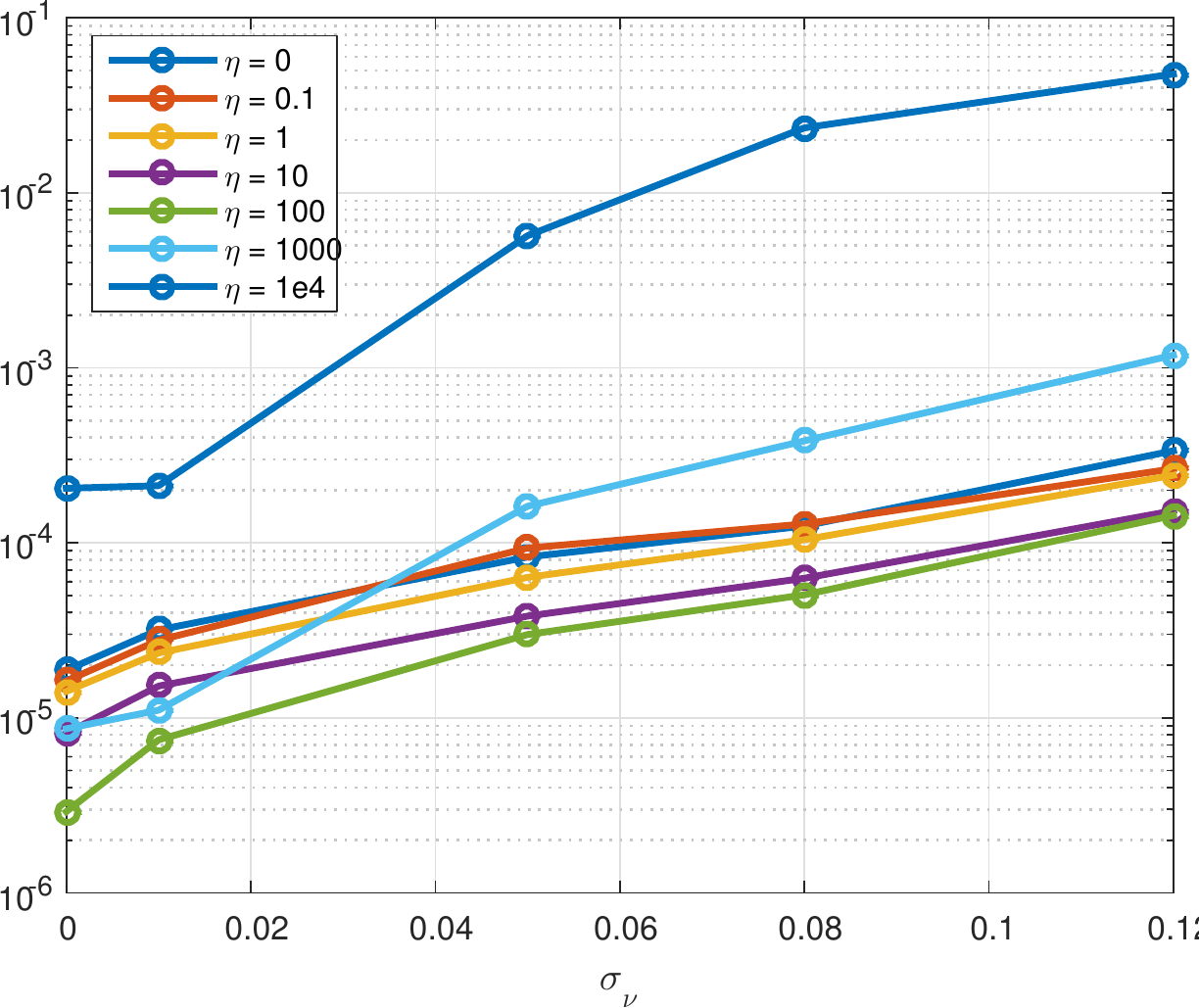}
	\caption{}
		\label{fig:dm_mse_nu}
				\end{subfigure}
	\caption{(a) MSE of the encoder for the training data, comparing 1 hidden layer (red 'x') vs. 2 hidden layers (blue circle) for varying $\eta$ values. (b) MSE of the encoder for the training data, comparing various $\eta$ values.}
	\label{fig:dm_mse}
\end{figure}

Since there is a great deal of randomness in the optimization, we trained encoders for 10 realizations of the data, and averaged the MSE over the realizations.
The encoder MSE for a single realization of the data is calculated as
\begin{equation*}
e^2 = \frac{1}{2m}\sum_{i=1}^m  \Vert o^e(x_i) - \Psi(i) \Vert^2.
\end{equation*}
Figure~\ref{fig:dm_mse}(a) plots the MSE for various values of $\eta$ for an encoder with 1 hidden layer and with 2 hidden layers.
In Fig.~\ref{fig:noisy_eta}, we plot several results from this experiment.
The top row displays the output of the encoder with one hidden layer and the bottom row with two hidden layers. 
Column (a) is the result without the added constraint. 
Columns (b)-(d) are for $\eta=10,\;\; 10^3,\;\; 10^6$ respectively.

A first conclusion from these simulations is that adding depth to the network improves the results.
This follows known observations in deep learning, that functions that can be compactly represented by networks with $L$ layers, can require an exponential number of units in a network with $L - 1$  layers~\cite{Bengio:2009}.
Using one hidden layer, the reconstruction of the embedding is noisy and the added constraint is not enough.
Using two hidden layers, the results are significantly improved and the eigenvector constraint serves to smooth out the noise in the reconstruction. 
In terms of the MSE, adding a layer decreases the error by a factor of 10.
\begin{figure*}[ht]
\centering
	\includegraphics[width=1.0\linewidth]{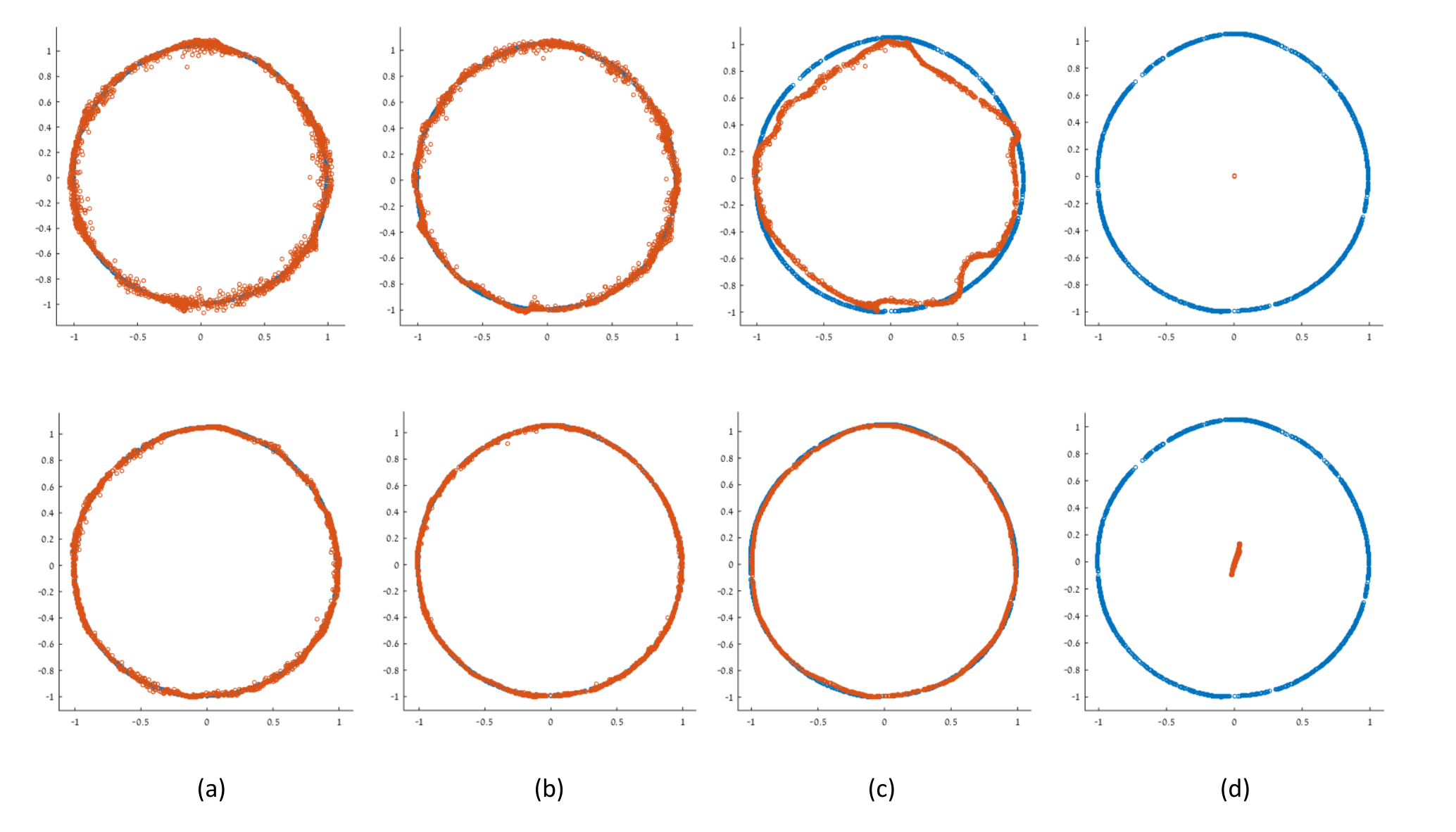}
	\caption{Top row: Encoder with 1 hidden layer, Bottom row: Encoder with 2 hidden layers. Blue plot is original diffusion map, red plot is output of the encoder. The columns examine the effect of the eigenvector constraint for increasing values of $\eta$.
	(a) $\eta=0$, (b) $\eta=10$ ,(c) $\eta=10^3$, (d) $\eta=10^6$. If $\eta$ is too large, the eigenvector constraint dominates the cost function and the output collapses to the trivial solution that all points equal zero.
	For a one hidden layer encoder, the eigenvector constraint is not enough to smooth out the noise and obtain a good embedding.
	For a two layer net, increasing $\eta$ within a reasonable range smooths the noise and yields an improved output compared to not including this constraint.}
		\label{fig:noisy_eta}
\end{figure*}
When the cost parameter is too large as in Fig.~\ref{fig:noisy_eta}(d), this constraint collapses the output of the encoder to the origin, which is the trivial solution to minimizing this constraint.
The optimal results were achieved for $\eta=10-100$.
Not imposing the EV constraint ($\eta=0$) yielded an error that was twice as high and results in a noisy approximation.

To examine the effect of noise on the approximation, we ran simulations with varying values of $\sigma_{\nu}$.
The encoder was trained with 2 hidden layers, with 20 hidden units in each layer.
We averaged the training MSE over 20 realizations of the data, where for each realization several encoders were trained, one for each value of $\eta$ in the EV constraint, displayed in Fig.~\ref{fig:dm_mse}(b). 
The best results are consistently obtained by $\eta=100$, regardless of the level of noise.
Even for no noise ($\sigma_{\nu}=0$), including the EV constraint with proper $\eta$ is meaningful for estimating the diffusion embedding ($\eta=0$ has second or third highest error for all $\sigma_{\nu}$ values).
If is $\eta$ is too high, the solution collapses to zero, resulting in a high error.
For low $\sigma_{\nu}$, $\eta=1000$ performs well so there is some trade-off between the SNR and how high $\eta$ can be.

To examine the dependence of the mean error on the number of units in the single layer encoder, we trained the network with a varying number of hidden units $s_1\in\{5,...,60\}$, for $\eta=100$, given clean data, i.e. $\sigma_{\nu}=0$. 
Figure~\ref{fig:layer1_bound} shows that the empirical mean error $e$ is bounded by $\frac{C}{s_1}$, so that empirically we are achieving a tighter bound than the bound guaranteed in Corollary~\ref{corr:convergence}.
\begin{figure}[ht]
\centering
	\includegraphics[width=0.4\linewidth]{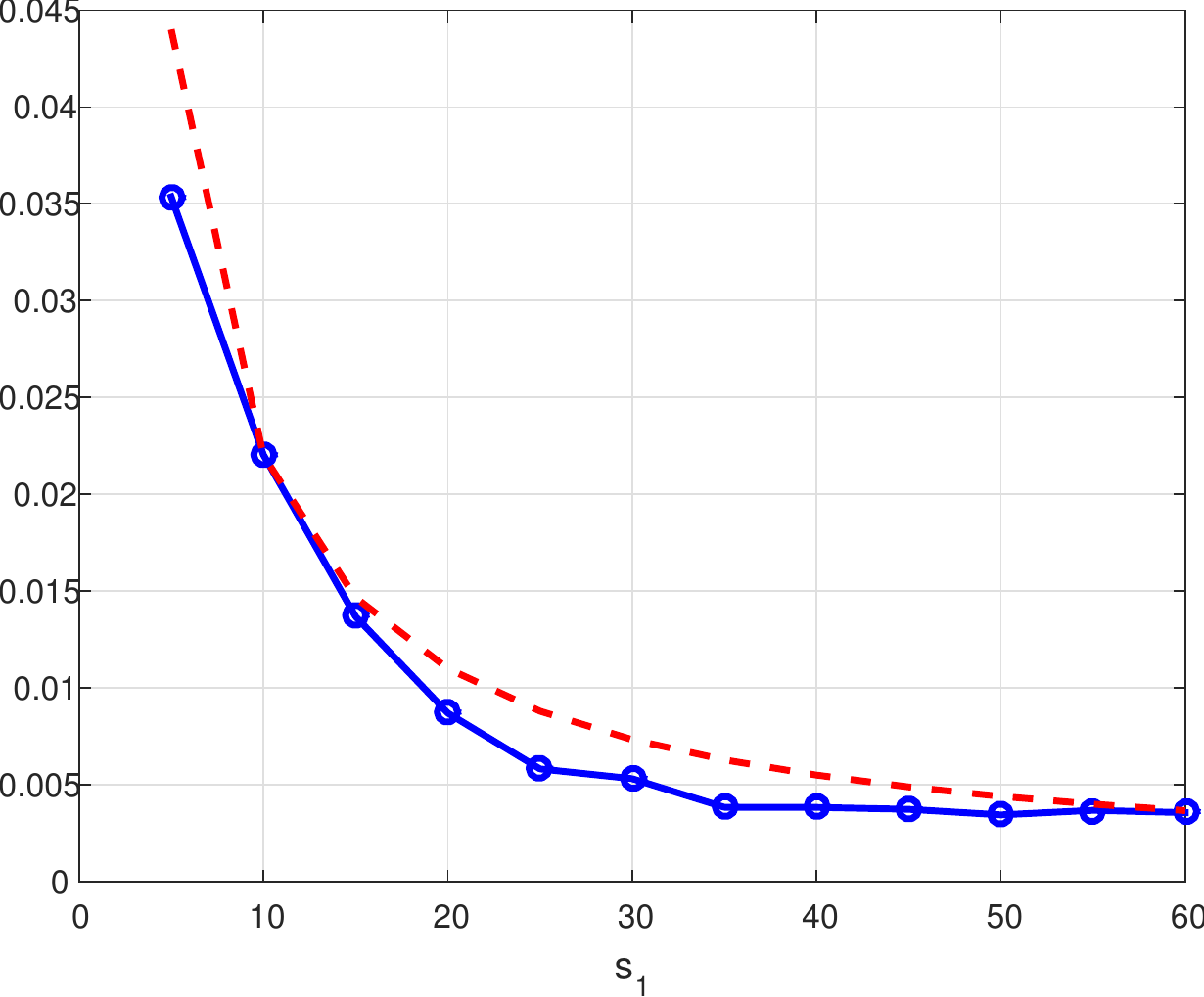}
	\caption{Embedding estimation error vs. number of hidden units $s_1$ in 1-layer encoder (solid line). 
	The empirical error is bounded by $\frac{C}{s_1}$ (dashed line), which is a tighter bound than the theoretical bound $\frac{C}{\sqrt{s_1}}$ guaranteed in Corollary~\ref{corr:convergence}.}
		\label{fig:layer1_bound}
\end{figure}

Regarding the parameters used in the simulations, the number of units in the hidden layers in the simulations was $s_l=20$, unless stated otherwise. 
The cost parameter of the weight regularization term was set as $\mu = 10^{-10}$. 
For the sparsity constraint in pre-training the auto-encoder~(\ref{eq:sparsity}), we tried various values of $\beta$ for $\rho = 0.1$ and they do not affect the errors in any significant manner. Therefore, we set $\beta=1$. 

\subsection{Decoder}
\label{subsec:results_decoder}
The decoder solves the pre-image problem.
It can be used to extend data from the diffusion embedding space to the data space, thus creating new points that lie on the manifold.
This enables a better visualization of the data belonging to the manifold, in the data space.

Our first example is based on the clean version of the 3D curve, given in~(\ref{eq:curve3d}).
We draw 2000 random samples from this curve. 
The diffusion map is calculated with $\sigma=0.1$ for the scale in the kernel~(\ref{eq:kernel}).
The first two diffusion coordinates are a circle, as in the previous example.
We train a decoder using these 2000 samples and their embedding.
We then cover the diffusion space  with points enclosed within a circle extended beyond the radius of the embedding and use the decoder to predict the data in the 3D data space.
Figure~\ref{fig:bubble} top-left and bottom-left display the points in the embedding space colored by angle and radius, respectively. 
The remaining three columns display different views of the predicted points in the 3D data space, where the colors correspond to the color of the points in the diffusion space.
We can see the points are restricted to a 2D surface enclosed within the 3D curve.
Note that the embedding of the origin is handled smoothly, with no singularities in the data space.
Also note that points in the diffusion space which are located beyond the original circle, are decoded in the data space along the boundary of the surface.
Thus, the range of extension is limited and follows the geometry of the original data.
\begin{figure*}[th]
\centering
	\includegraphics[width=1\linewidth]{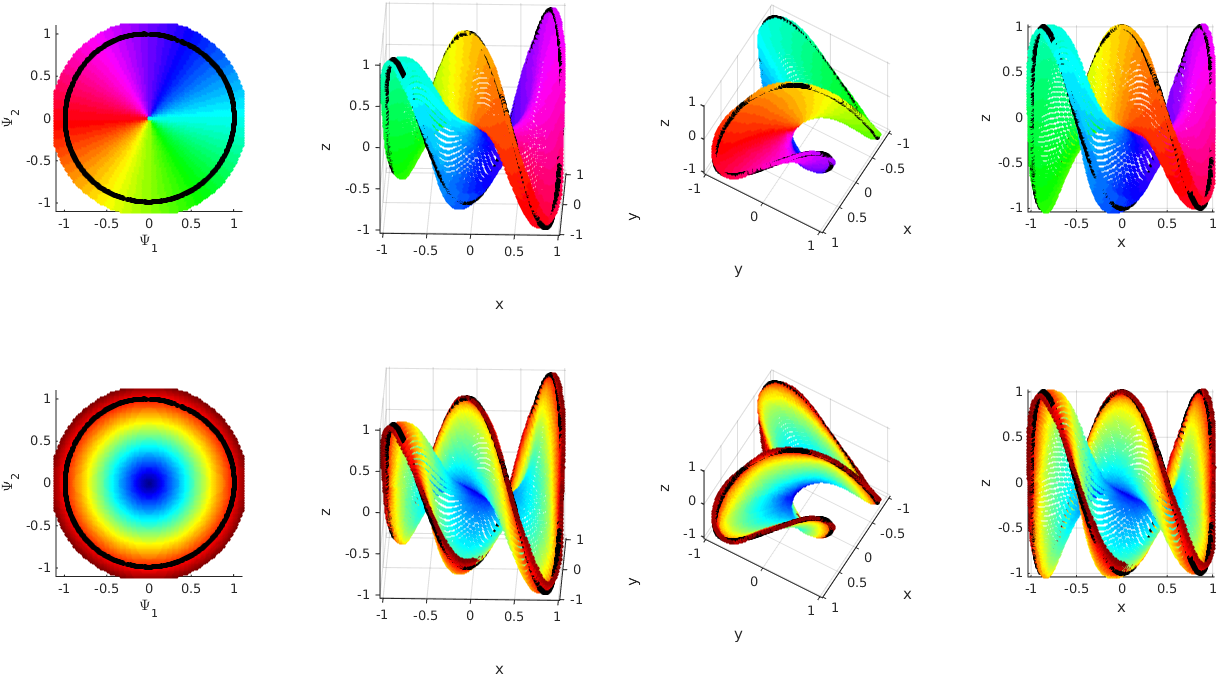}
	\caption{Left column: the diffusion space is covered in points within a circle. The black points are the diffusion coordinates for the training data and are the input points to the decoder. The remaining three columns are different views of the output of the decoder, where the black points are the original training datapoints. Top row: points colored by angle in the diffusion space. Bottom row: points colored by radius.}
	\label{fig:bubble}
\end{figure*}

The next example is of more complex data, whose embedding is also a 2D circle.
Given an image of a noisy periodic function (std. of the noise is 25), shown in Fig.~\ref{fig:periodic_im}(a), we extract all overlapping patches sized $5 \times 5$ pixels, and construct a random walk on the patches, to obtain the diffusion map $\Psi$.
Figure~\ref{fig:periodic_im}(b) displays examples of 49 patches extracted from the image.
The first two coordinates of the diffusion embedding are a circle shown in Fig.~\ref{fig:periodic_im}(c).
Training the a decoder with two hidden layers for the diffusion map and the patches, yields a function that ``produces" image patches.
\begin{figure}[t]
\centering
	\includegraphics[width=0.9\linewidth]{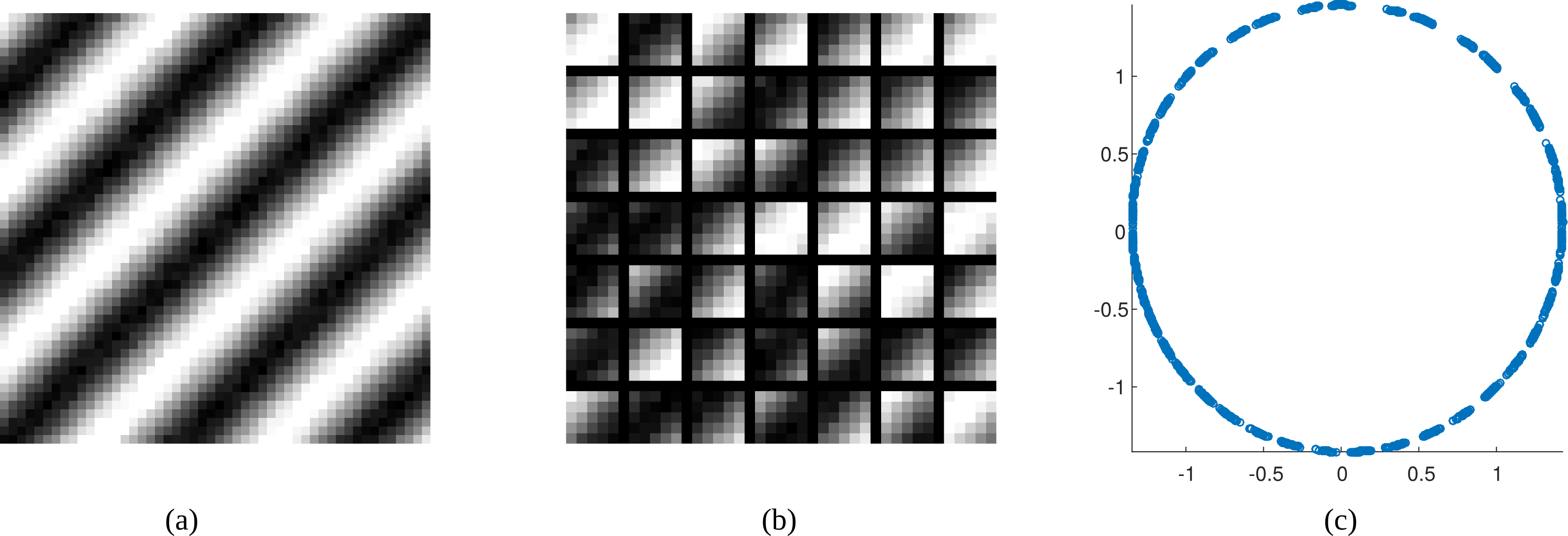}
	\caption{(a) Image of noisy periodic function. (b) Example of 49 training patches extracted from the image. (c) Diffusion map of training image.}
		\label{fig:periodic_im}
\end{figure}
The input dimension is $s_1=2$ and the output dimension is $s_L=25$.
The decoder enables visualization of the data.
Figure~\ref{fig:decoder_image}(a) displays the patches obtained by inserting several test points from the diffusion space into the decoder, where the position of the patches in the circle corresponds to the location of the points entered into the decoder.
This example shows that the radius in the diffusion embedding represents the amplitude of the periodic function in the image patches. 
At the origin of the diffusion space, you get a smooth patch.
Note the patches are clean as opposed to the training patches. 

Figure~\ref{fig:decoder_image}(b) shows how this can be used for image manipulation. 
A rotation of $\pi$ and scale by 0.45 is applied to the diffusion map as
\begin{equation*}
\widetilde{\Psi} = 0.45 \begin{pmatrix}
  -1 & 0  \\
  0 & -1 
 \end{pmatrix}
\Psi.
\end{equation*}
The values of $\widetilde{\Psi}$ are inputted into the decoder. 
Then the output values are used to reconstruct the image where each value is assigned to its original pixel location.
The resulting image is indeed a shift of the original image in Figure~\ref{fig:periodic_im} and the amplitude has been decreased.

\begin{figure}[th]
\centering
\begin{subfigure}[b]{0.4\textwidth}
	\includegraphics[width=\textwidth]{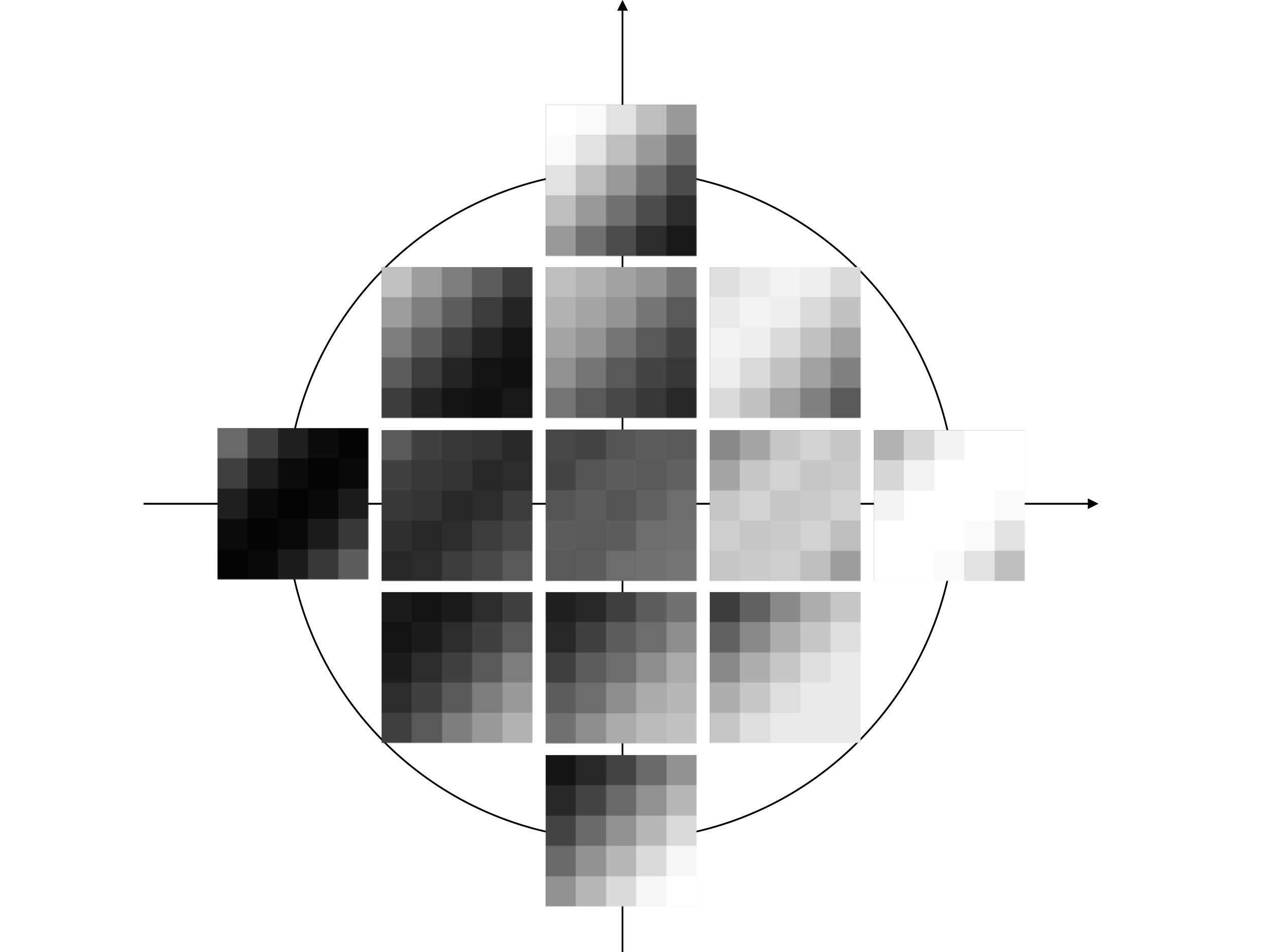}
	\caption{}
		\label{fig:patches}
		\end{subfigure}
	\begin{subfigure}[b]{0.25\textwidth}
	\includegraphics[width=\textwidth]{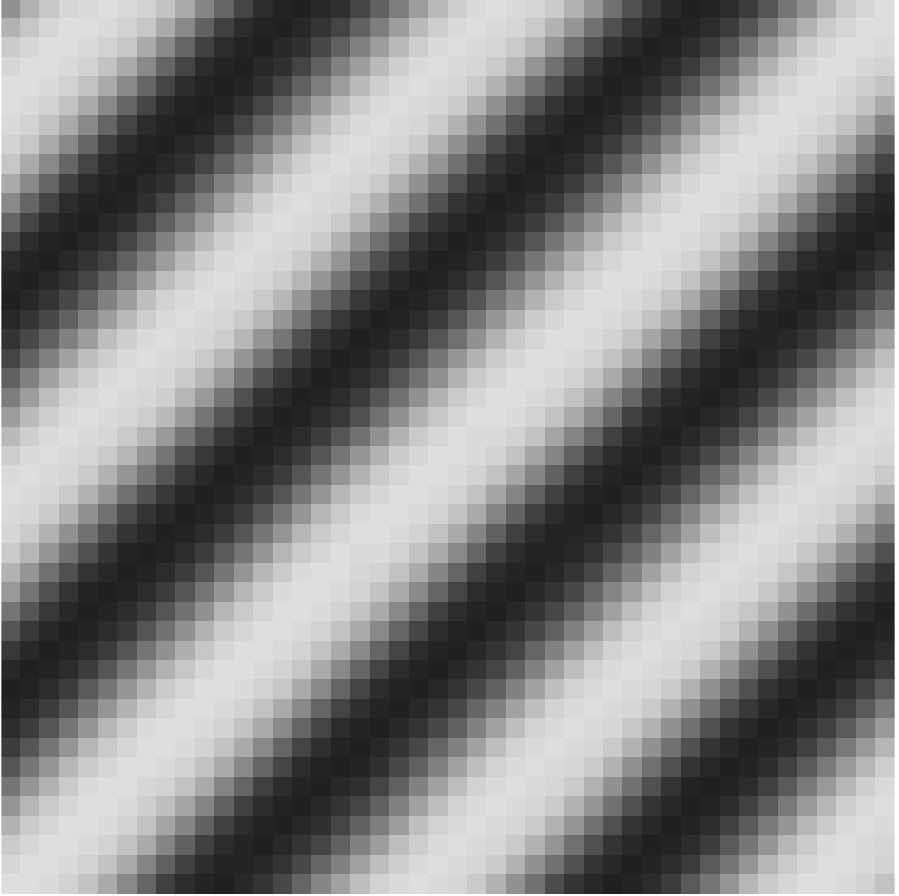}
	\caption{}
		\label{fig:rotIm}
		\end{subfigure}	
	\caption{(a) Decoding image patches. Various points along and within the circle are inputted to the decoder. We display the obtained patches at the locations in the diffusion space that were used an input to the decoder.
	This demonstrates that the radius of the embedding corresponds to the amplitude of the periodic function in the image space.
	(b) Image manipulation: the training diffusion map is rotated by 180 degrees and the values multiplied by 0.45. Inputting these points into the decoder and reconstructing the image from output, shows the period of the image has indeed shifted and the amplitude decreased.}
			\label{fig:decoder_image}
\end{figure}

\subsection{Autoencoder}
We train an autoencoder for the noisy 3D data described in Sec.~\ref{res:encoder}.
First, we train a decoder from the embedding to the noisy data, using either one or two hidden layers, corresponding to what was used in the encoder.
We then stack the decoder on top of the trained encoders, and obtain an autoencoder.
We sample 1000 new test points from the noisy curve, and calculate their reconstruction via the autoencoders.
These training and test phases were repeated for 10 realizations of the data.
For each realization of the data, we train several encoders for different values of $\eta$ and one decoder.
Thus, the MSE of the autoencoder depends on $\eta$ and not on the optimization of the decoder, i.e. a lower cost is due to the encoder since all autoencoders share the same decoder.
We then average the reconstruction MSE over all realizations, for each $\eta$ separately.
The reconstruction error is given by
\begin{equation*}
\epsilon = \frac{1}{m} \sum_{i=1}^m \Vert x_i - r(x_i)\Vert^2.
\end{equation*}
Note that we are comparing the output of the autoencoder to the \emph{clean} data $\{x_i\}_i$, although the autoencoder is trained using the noisy data $\{\tilde{x}_i\}_i$.
This demonstrates the denoising capabilities of the network.
Figure~\ref{fig:data_mse}(a) summarizes the results for various values of $\eta$ for autoencoder with 1 hidden layer, and with 2 hidden layers.
Adding a layer decreases the reconstruction MSE by a factor close to 10.
\begin{figure}[t]
\centering
\begin{subfigure}[b]{0.45\textwidth}
	\includegraphics[height=2in]{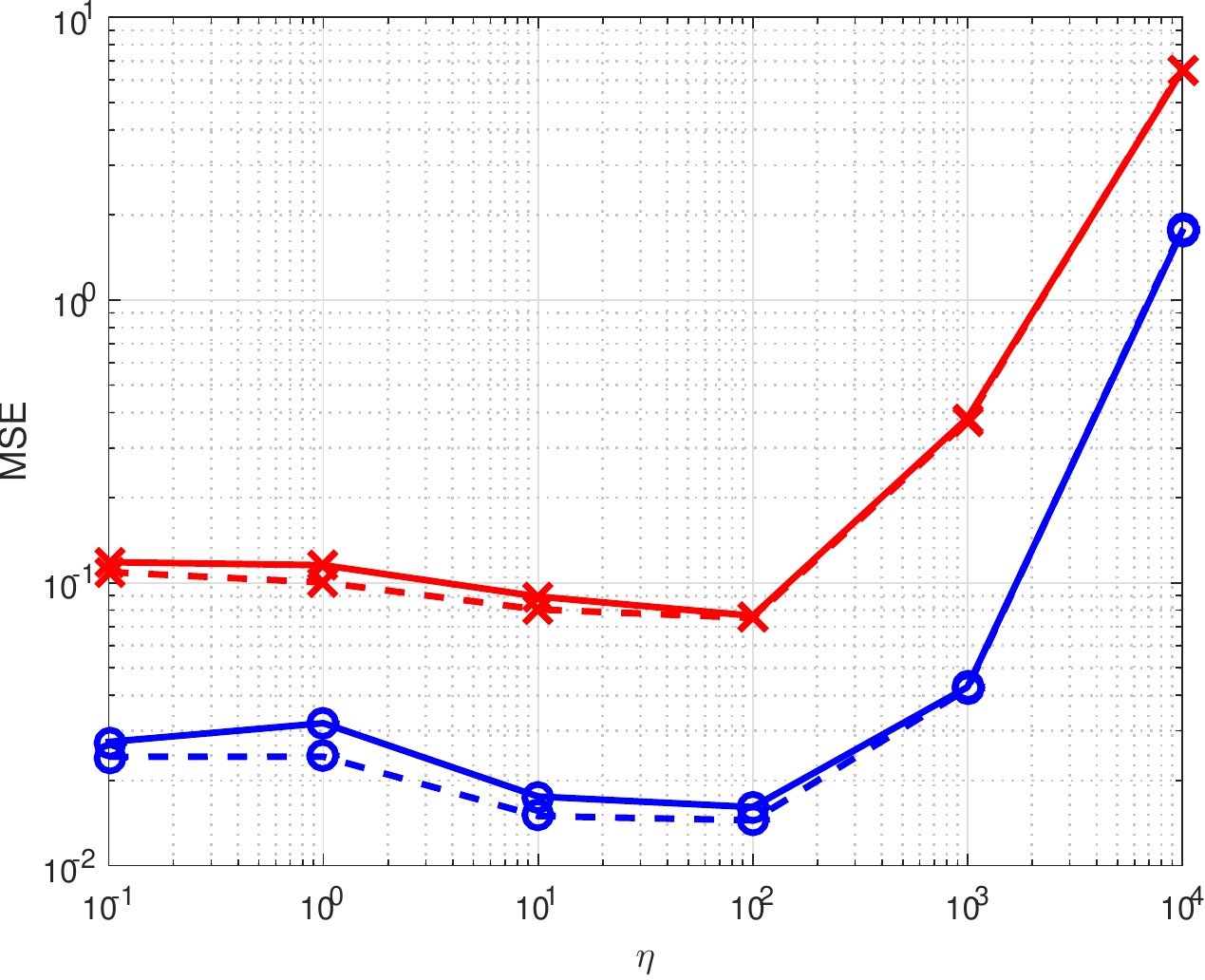}
	\caption{}
		\label{fig:data_mse_1vs2}
		\end{subfigure}
\begin{subfigure}[b]{0.45\textwidth}
	\includegraphics[height=2in]{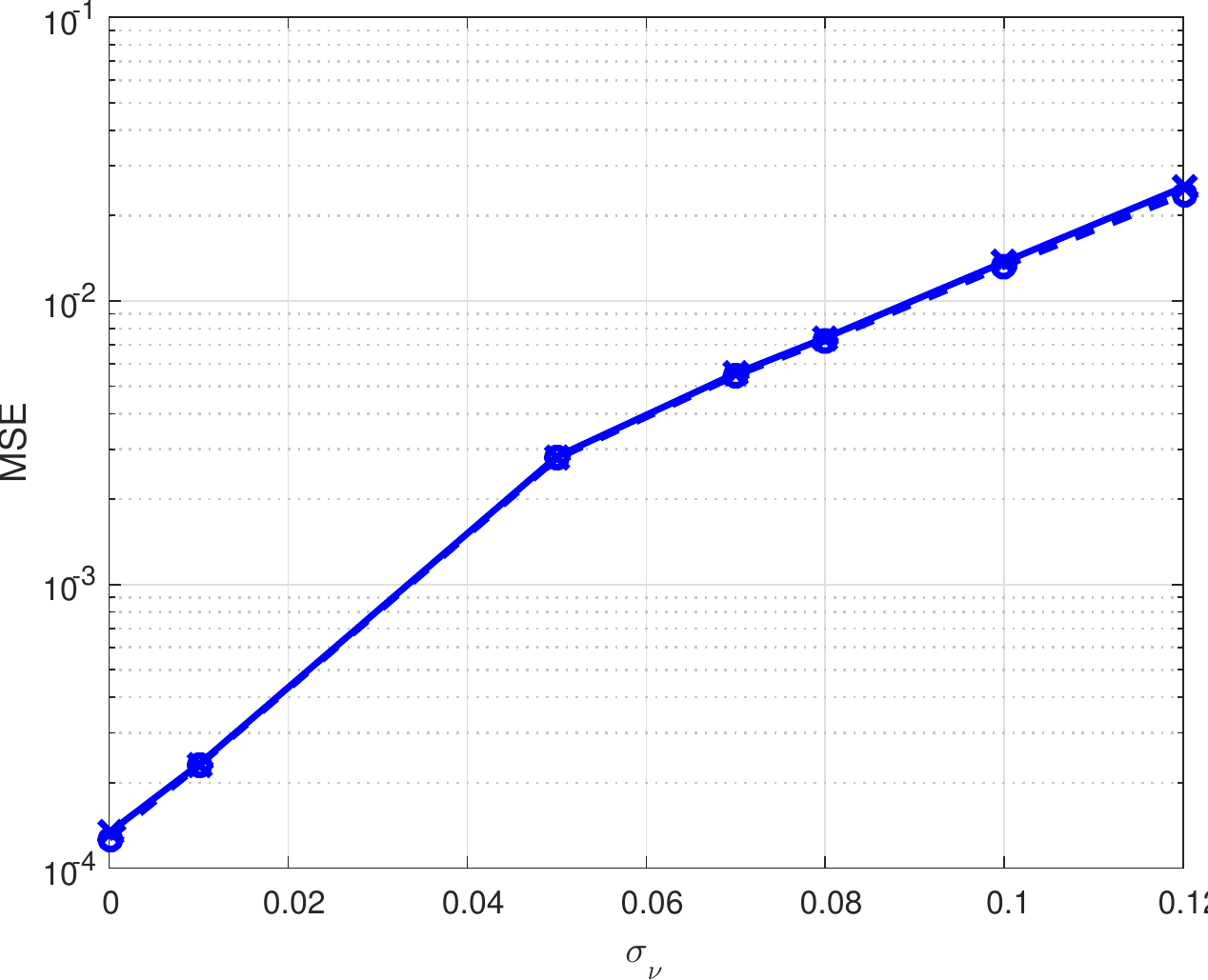}
	\caption{}
		\label{fig:data_mse_nu}
		\end{subfigure}
	\caption{(a) MSE of data reconstruction via the autoencoder, comparing 1 hidden layer (red 'x') vs. 2 hidden layers (blue circle) for varying $\eta$ values. Dashed line is training data, solid line is test data.
	(b) MSE of data reconstruction vs the std. of the noise. Reconstruction is performed via the autoencoder for $\eta=100$ in the EV constraint. Solid line is training error, dashed line is test error.}
	\label{fig:data_mse}
\end{figure}
We can see that in all simulations, the training and test reconstruction errors in the data space are of the same order, so that the algorithm does not over-fit the training data.
In addition, using 2 layers in the network improves both encoder and autoencoder errors. 
In addition, as previously shown, if $\eta$ is too large, for example $\eta=10^6$, this corrupts both the embedding and the data.
The best overall results were achieved for $\eta=100$.

To examine the effect of noise on the diffusion net, we trained the autoencoder for varying values of $\sigma_{\nu}$.
The network was trained with 2 hidden layers in the encoder, and 2 hidden layers in the decoder, with 20 hidden units in each layer. 
Following the experimental results of the encoder, we set $\eta=100$ in the EV constraint.
We average the training and test reconstruction error over 20 realizations of the data, displayed in Fig.~\ref{fig:data_mse}(b). 
The reconstruction error is between the reconstructed output and the original clean data, demonstrating the denoising capabilities of the diffusion net for increasing noise std.
As in the previous example, the training and test errors are very similar, implying that the training does not over-fit the data, both in the encoder and in the decoder, even in increasingly noisy scenarios.

\subsection{Outlier Detection}
We now apply the autoencoder to real image data and demonstrate that the autoencoder performs outlier detection.
As stated in Sec.~\ref{sec:diffnet}, test data will not necessarily follow the model of the data used to calculate the embedding.
OOSE applied to new data that significantly differs from the training data will assign embedding values that do not distinguish it from the data.
Therefore, it is important to be able to determine when the embedding is unreliable.

Figures~\ref{fig:Holes1_ID52}(a) and~\ref{fig:T_shell_beta_ID10}(a) display two images of patterned semiconductor wafers acquired by a scanning electron microscope, sized $200 \times 200$ pixels. 
Both wafer images have a defect near the center of the image.
For each image separately, we randomly extract 2500 patches from the image, sized $8 \times 8$ pixels.
This training set is used to calculate a diffusion map, reducing the data from dimensionality $n=64$ to $d=2$ dimensions.
The training set is used to train an autoencoder, as in Algorithm~\ref{alg:ae}, and the average training reconstruction error $\epsilon$~(\ref{eq:ae_MSE}) is calculated.
We then input all overlapping image patches from the image into the autoencoder and calculate the reconstruction error of each patch.
Figures~\ref{fig:Holes1_ID52}(b) and~\ref{fig:T_shell_beta_ID10}(b) display $\Vert r(x') - x' \Vert / \epsilon$ for all pixels in the image, revealing that this approach easily separates the defects from the background.

This is a result of the diffusion map capturing the main structure of the data, i.e. the pattern of the wafer, as represented in the training data.
Patches which differ from the training set, as in the case of the defects, are not represented in the diffusion map.
Thus, when applying the autoencoder, these patches are not properly reconstructed by the mappings learned from the data space to the diffusion space and back.
This result obtained by the autoencoder indicates that, for these patches, the embedding provided by the encoder does not properly represent them.
\begin{figure}[th]
\centering
\begin{subfigure}[b]{0.4\textwidth}
	\includegraphics[height=4.7cm]{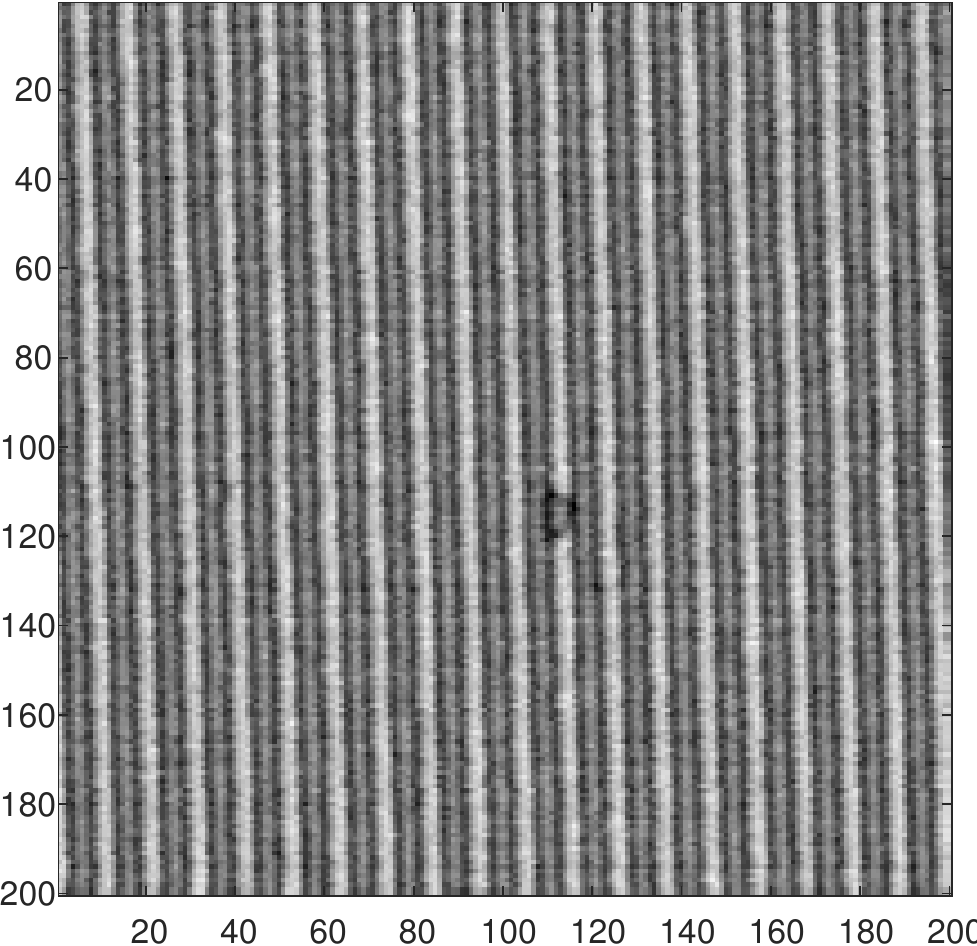}
	\caption{}
		\end{subfigure}
\begin{subfigure}[b]{0.4\textwidth}
	\includegraphics[height=4.75cm]{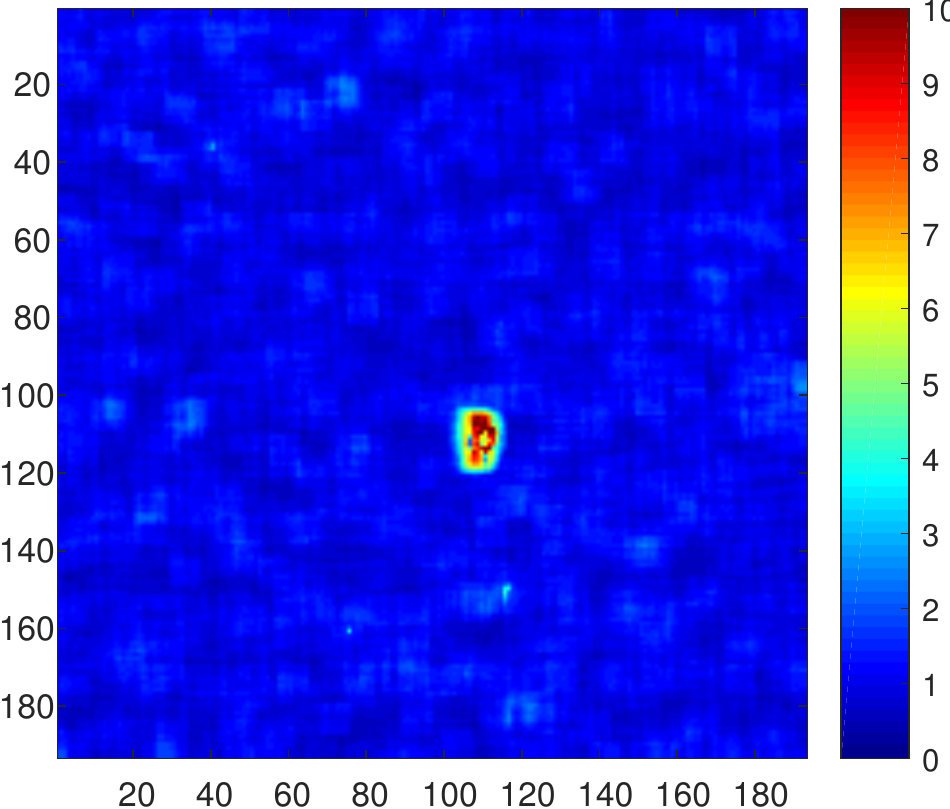}
	\caption{}
		\end{subfigure}
	\caption{(a) SEM image of a semiconductor wafer with defect near the middle of the image. (b) Reconstruction error of the image relative to the average training error: $\Vert r(x') - x' \Vert / \epsilon$. This reveals the wafer defect.}
	\label{fig:Holes1_ID52}
\end{figure}

\begin{figure}[th]
\centering
\begin{subfigure}[b]{0.4\textwidth}
	\includegraphics[height=4.7cm]{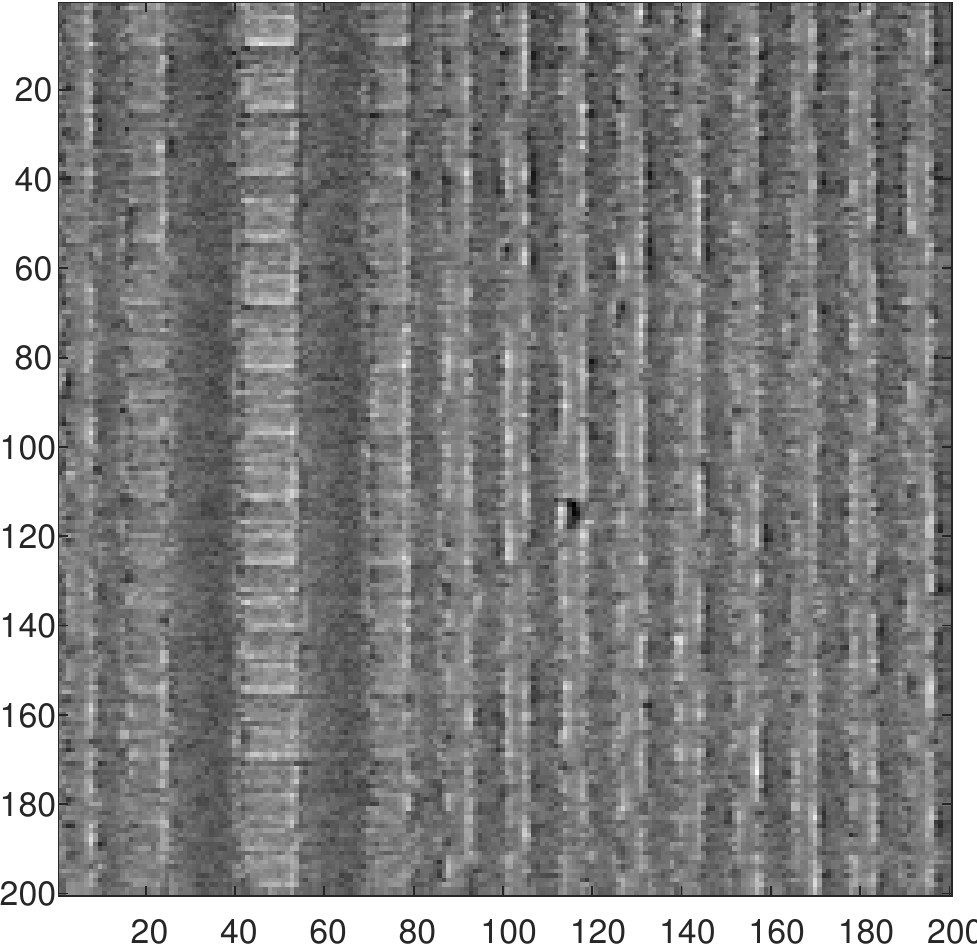}
	\caption{}
				\end{subfigure}
\begin{subfigure}[b]{0.4\textwidth}
	\includegraphics[height=4.765cm]{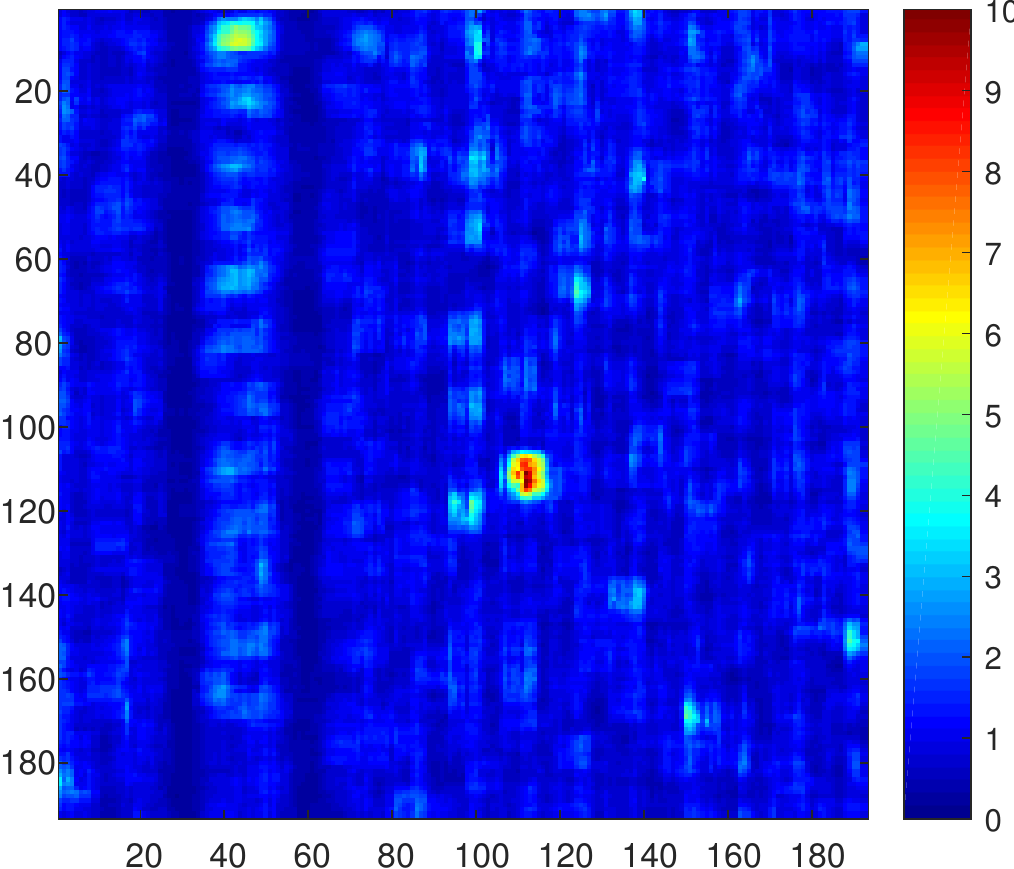}
	\caption{}
		\end{subfigure}
	\caption{(a) SEM image of a semiconductor wafer with defect near the middle of the image. (b) Reconstruction error of the image relative to the average training error: $\Vert r(x') - x' \Vert / \epsilon$. This reveals the wafer defect.}
			\label{fig:T_shell_beta_ID10}
	\end{figure}

\subsection{OOSE Error vs. Variation in In-sample Embedding}
The purpose of an out-of-sample extension algorithm in the manifold learning setting is to provide an extension of the embedding to new points, such that this extension on the new points is close to the embedding of these points, if the embedding was calculated over all the points.
Due to the discrete nature of the data, there is no ``true'' embedding; rather the value of the eigenvectors for a given training point depends on the other points in the set and the scale used in the affinity matrix.
If the eigenvalues of the decomposed matrix have a geometric multiplicity, than the eigenvectors belonging to the same eigenvalue span a subspace, and for the same data they are similar up to a rotation. 
In our experiments in Sec.~\ref{res:encoder}, the data is closed curve, so that its continuous equivalent is the heat equation with Neumann boundary condition.
Thus continuous solutions to this PDE are the trigonometric functions $\sin(\cdot), \cos(\cdot)$.
For our discrete 3D data, the eigenvectors approximate these continuous periodic eigenfunctions.
The first two eigenvectors belong to the same eigenvalue and form a 2D circle in the embedding space.

To demonstrate the performance of our OOSE approach, we conducted an experiment similar to that proposed in~\cite{Bengio2003}.
We took $m=2000$ training points along the curve and calculated the diffusion embedding $\Psi$ for these points.
We then added $n_\textrm{test}$ points to the training points, and calculated am embedding $\widetilde{\Psi}$ for all $m+n_\textrm{test}$ points.
Since the embedding in both cases is a circle, we calculated the rotation between both embeddings using the shared points, i.e. the training points.
This was calculated by~\cite{Coifman:2014} 
\begin{equation*}
S[i,j]=\sum_k \widetilde{\Psi}_i^T(k)*\Psi_j(k), \;\;\;\; k\in\{1,...,2000\} \; i,j={1,2}.
\end{equation*}
The SVD of S is $U \Lambda V^T$.
Then the rotation from $\widetilde{\Psi}$ to $\Psi$ is given by
\begin{equation*}
R=VU^T.
\end{equation*}
We then calculated the error
\begin{equation*}
\epsilon_{n_\textrm{test}}=\frac{1}{m} \sum \Vert \Psi(k)-R\widetilde{\Psi}(k) \Vert, \;\;\;\; k\in\{1,...,2000\} .
\end{equation*}
We calculated this error over 10 realization of the data and for 4 $\sigma_{\nu}$ values.
This error is the variation in the in-sample embedding due to increasing the number of points for which the embedding is calculated.
Figure~\ref{fig:map_vs_encoder} compares this error to the error achieved by our best encoder, trained using $m=2000$ points and $\eta=100$.
\begin{figure}[th]
\centering
	\includegraphics[width=0.5\linewidth]{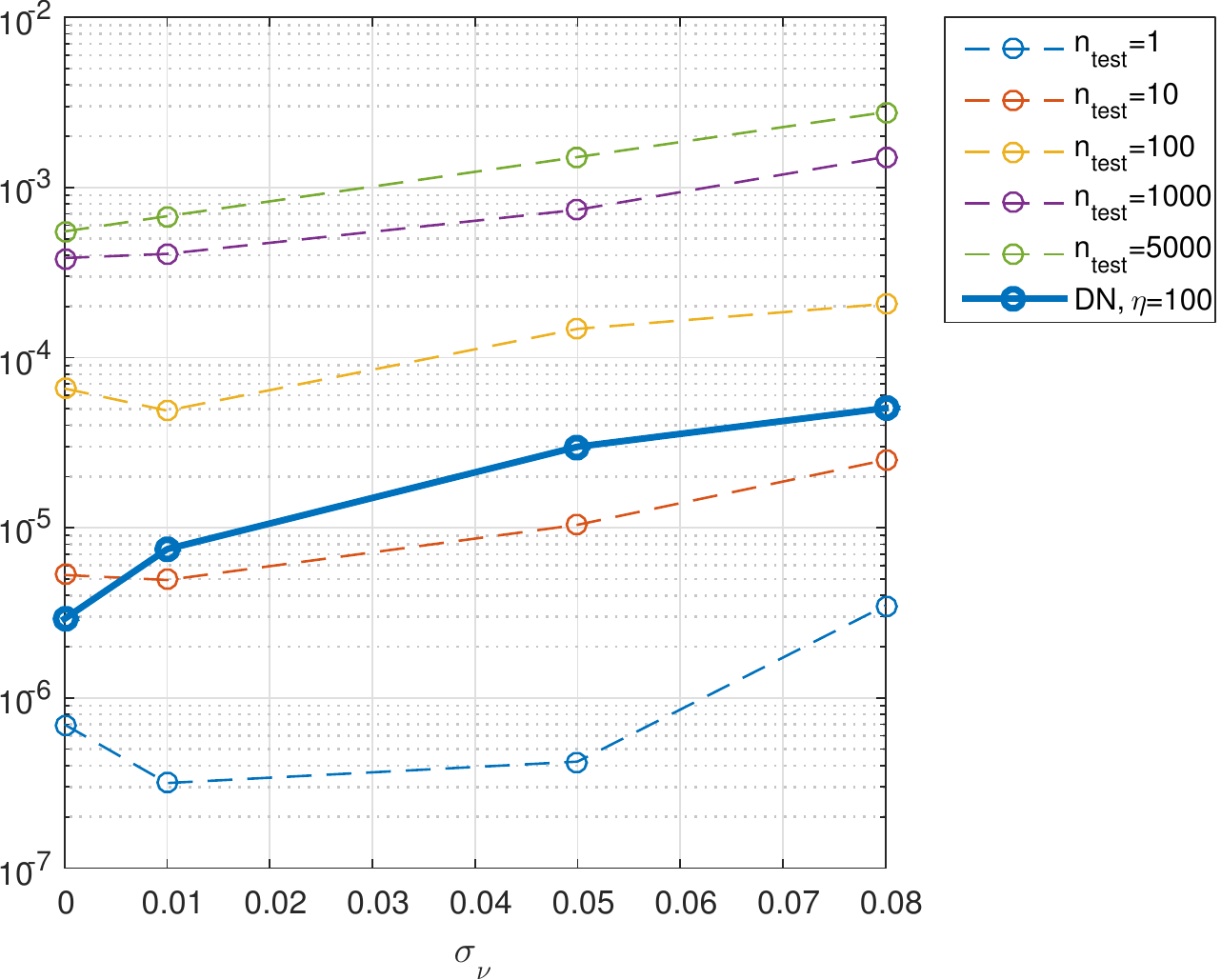}
	\caption{Dashed plot is $\epsilon_{n_\textrm{test}}$, solid line is DN encoder training error.}
		\label{fig:map_vs_encoder}
\end{figure}
For $n_\textrm{test}>100$, the out-of-sample extension error of our encoder is lower than the variation in the in-sample embeddings.
Since, OOSE will typically be performed for $n_\textrm{test}$ much higher than 100, this demonstrates that our method provides a good out-of-sample extension.

\section{Conclusions and Future Work}
\label{sec:future}
We have presented a new framework employing deep learning for manifold learning.
We proposed designing an encoder and decoder that learn the mapping between a given high-dimensional dataset and low-dimensional embedding, and vice-versa.
To this end, we proposed a new constraint in training the encoder, which preserves the locality of the points in the embedding.
We demonstrated empirically that this constraint improves the approximation of the embedding.
Our encoder enables very efficient out-of-sample extension of the non-linear embeddings to new points, with low memory costs. 
The decoder provides a solution to the pre-image problem, enabling data visualization and augmentation.
Finally, stacking the two networks together as a deep autoencoder enables both denoising and outlier detection of the data, as seen via the embedding.
Calculating the reconstruction error of the autoencoder for new points allows to evaluate whether the OOSE provided by the encoder properly represents these new points.
We presented experimental results in noisy scenarios for simulated and real data, demonstrating the properties of the proposed architecture.

Our main focus has been on the encoder for performing out-of-sample extension for data whose distribution follows the distribution of the training data.
However, in different applications, such as sequential signal processing, the nature of the data can change over time.
In manifold learning, the embedding is usually calculated once for training points, and does not adapt over time for new points, as opposed to online dictionary learning in sparse representations, for example.
This could lead to the embedding not providing a good representation of the data as it evolves, and requires re-calculating the embedding again and again.
In future work, we propose to develop a method based on online fine-tuning of the autoencoder that will adapt the embedding to new points which do not fit the model of the training data.
Instead of performing ``regular'' fine-tuning of the autoencoder, constraints can be added that will maintain the middle layer as an approximation of the embedding, as we proposed with the encoder in this work.
In this case, we will fine-tune with both the test and training data, where the training data regularizes the autoencoder so that its middle layer remains an approximation of the embedding for the training points. 
By fine-tuning the network so that it reconstructs the new test data, the middle layer should recover a new embedding for the test data.
This adaptive approach will be explored in future work.

A second direction is to further explore the decoder and how including different regularizations affects the solution of the pre-image problem.
Including a harmonic constraint for example should enable recovering a minimal surface as the example shown in Sec.~\ref{subsec:results_decoder}.
The error rate we provided on the encoder does not apply to the decoder as it requires the function that is being approximated to be band-limited, which does not hold for the decoder in a general case.
In future work on the decoder, we intend to provide a theoretical analysis of the decoder, and to expand our theoretical results to multi-layer nets.
Computing the pre-image is important in different applications in which interpolating the data by averaging in the high-dimensional data space is meaningless, such as the possibility of performing image texture synthesis.
We will analyze datasets in which the high-dimensional data is more complicated and examine how this affects the required complexity of the decoder architecture. 

A third research direction is to examine improving deep learning applications.
Our network enables to determine the number of nodes needed to learn the geometry of the data and can be used to infer the maximal number of nodes needed to model the complexity of the system for an unconstrained neural net.
In addition, we intend to explore whether incorporating our encoder into a deep network will improve deep neural networks. 
This is motivated by previous works that have shown that implicitly incorporating the manifold assumption in the construction of deep networks improves classification results. 
Therefore, we expect that explicitly including the embedding in the network via the encoder should be beneficial.

\section*{Acknowledgments}
This research was supported by the Israel Science Foundation (grant no. 1130/11).
Alexander Cloninger is supported by NSF Award No. DMS-1402254.
The authors would like to thank Ronald Coifman, Ronen Talmon and Roy Lederman for helpful discussions and suggestions.

\appendix
\section{Proof of Theorem \ref{thm:convergence}} 
\label{app:proof}
\begin{proof}

The proof of the theorem requires two key results in the literature, relying on theorems by Barron~\cite{Barron1993} and Coifman and Lafon~\cite{Coifman2006a}.
\begin{theorem}{(Theorem 1 from \cite{Barron1993})}\label{thm:Barron} 
Let $f:\R^n\rightarrow \R$ be a function with bounded first moment of its Fourier transform
\begin{eqnarray*}
\int_{\R^n} |\widehat{f}(\omega)||\omega| d\omega \le C_f.
\end{eqnarray*}
Let $B_r\subset\R^n$ be a Euclidean ball around zero with radius $r$, which we assume contains our data.  Then for every $n\in\N$ there exists a linear combination $f_K$ of $K$ sigmoidal units such that 
\begin{eqnarray}
\left(\int (f(x) - f_K(x))^2 dx \right)^{\frac{1}{2}} \le \frac{C_1}{\sqrt{K}},
\end{eqnarray}
where $C_1 = 2\pi C_f \mu(B_r)$, and $\mu$ is the Lebesuge measure on $\R^n$.
\end{theorem}

Now it suffices to show that the extension function $f$ of $\psi$ from \eqref{eq:extender} has a bounded first moment of its Fourier transform.  To show this, we rely on a second result.

\begin{theorem}{(Proposition 11 from \cite{Coifman2006a})}\label{thm:lafon} 
Let $\cal{M}$ be a submanifold in $\R^n$ and $\psi:\Gamma\rightarrow\R$ be an eigenfunction of its Laplacian with eigenvalue $\lambda$.  Let $\delta>0$ be an approximation level.  Let $f$ be an extension function as in \eqref{eq:extender}.  Then there exists a band limited function $b:\R^n\rightarrow\R$ with band $C_{\delta,\psi}\lambda$ such that 
\begin{eqnarray*}
\frac{\int_{\R^n} (f(x) - b(x))^2 dx}{\int_{\R^n} (f(x))^2 dx} < \delta.
\end{eqnarray*}
\end{theorem}

We need several more intermediate claims before addressing the result.
\begin{claim}\label{claim:fBounded}
The function $f$ from \eqref{eq:extender} is in $L^2(\R^n)$.
\end{claim}
\begin{proof}[Proof of Claim \ref{claim:fBounded}]
Since $\rho$ is locally bi-Lipschitz,
\begin{eqnarray*}
\int_{\R^n} f(x)^2 dx &=& \int_{\R^n} e^{-2\lambda \|x - P_{\cal{M}}x\|^2_2} \psi(P_{\cal{M}}x)^2 dx\\
&\le& \int_{\cal{M}} \left( \int_{\R^{n-d}}e^{-2\lambda \|x - x'\|^2_2} dx \right) \psi(x')^2 dx'\\
&=& C_{n-d,\lambda}  \int_{\cal{M}} \psi(x')^2 dx'\\
&\le& \frac{C_{n-d,\lambda} }{1-\epsilon} \int_{\cal{M}} \psi(z)^2 d\rho(z)\\
&=& \frac{C_{n-d,\lambda} }{1-\epsilon} 
\end{eqnarray*}
\end{proof}

\begin{claim}\label{claim:bBounded}
Let $b$ be as in Theorem \ref{thm:lafon}, and $\widehat{b}$ be its Fourier Transform.  Then $b\in L^2(\R^n)$ and 
\begin{eqnarray*}
\int |\widehat{b}(\omega)||\omega|d\omega <\infty.
\end{eqnarray*}
\end{claim}
\begin{proof}[Proof of Claim \ref{claim:bBounded}]
Clearly $b\in L^2(\R^n)$ since 
\begin{eqnarray*}
\|b\|_{L^2} \le \|f\|_{L^2} + \|b-f\|_{L^2}.
\end{eqnarray*}

Because $b$ is band limited with band $C_{\delta,\psi}\lambda$, meaning supp$(\widehat{b})$ is contained inside a ball of radius $C_{\delta,\psi}\lambda$, then
\begin{eqnarray*}
\int |\widehat{b}(\omega)||\omega|d\omega & = & \int_{B_{C_{\delta,\psi}\lambda}} |\widehat{b}(\omega)||\omega|d\omega \\
&\le & \left( \int (\widehat{b}(\omega))^2d\omega \right)^{\frac{1}{2}} \left( \int_{B_{C_{\delta,\psi}\lambda}} |\omega|^2 d\omega\right)^{\frac{1}{2}} \\
&=& \left( \int (b(x))^2 dx \right)^{\frac{1}{2}} \left( \int_{B_{C_{\delta,\psi}\lambda}} |\omega|^2 d\omega\right)^{\frac{1}{2}} \\
&\le& \left( \|f\|_{L^2}  +  \sqrt{\delta}\|f\|_{L^2} \right)  \left( \int_{B_{C_{\delta,\psi}\lambda}} |\omega|^2 d\omega\right)^{\frac{1}{2}} \\
&=&  \left( \|f\|_{L^2}  +  \sqrt{\delta}\|f\|_{L^2} \right) \left( \frac{n}{n+2}(C_{\delta,\psi} \lambda)^2 \mu(B_{C_{\delta,\psi}\lambda}) \right)^{\frac{1}{2}}\\
&=& \left( 1+\sqrt{\delta}\right) \sqrt{\frac{n \mu(B_1)}{n+2}} \left( C_{\delta,\psi} \lambda \right)^{\frac{n+2}{2}} \|f\|_2,
\end{eqnarray*}
where $\mu(B_1)$ is the volume of a ball of radius $1$ in $\R^n$.
\end{proof}

Now we prove Theorem~\ref{thm:convergence}.
By setting $\delta = \frac{1}{K}$ in Theorem \ref{thm:lafon}, we combine the above results to show
\begin{eqnarray*}
\left(\int (f(x) - f_K(x))^2 dx \right)^{\frac{1}{2}} &\le& \left(\int (f(x) - b(x))^2 dx \right)^{\frac{1}{2}} + \left(\int (b(x) - f_K(x))^2 dx \right)^{\frac{1}{2}}\\
&\le& \frac{C}{\sqrt{K}},
\end{eqnarray*}
where 
\begin{eqnarray*}
C &=& \|f\|_2 + C_1 \\
&=& \|f\|_2 + 2\pi \mu(B_r) C_b \\
&\le& \left(1 + 2\pi \mu(B_r) \left( 1+\sqrt{\delta}\right) \sqrt{\frac{n \mu(B_1)}{n+2}} \left( C_{\delta,\psi} \lambda \right)^{\frac{n+2}{2}} \right) \|f\|_2\\
&\le& \left(1 + 2\pi \mu(B_r) \left( 1+\sqrt{\delta}\right) \sqrt{\frac{n \mu(B_1)}{n+2}} \left( C_{\delta,\psi} \lambda \right)^{\frac{n+2}{2}} \right)   \sqrt{\frac{C_{n-d,\lambda} }{1-\epsilon}}. 
\end{eqnarray*}

\end{proof}

%\clearpage

\footnotesize{
\bibliographystyle{IEEEtran}
\bibliography{mybib}
}

\vfill\pagebreak
\end{document}